\documentclass[10pt,twocolumn,letterpaper]{article}
\usepackage[pagenumbers]{cvpr}
\usepackage{microtype}
\usepackage[utf8]{inputenc}
\usepackage[T1]{fontenc}

\definecolor{cvprblue}{rgb}{0.21,0.49,0.74}
\definecolor{dtlorapurple}{RGB}{238,230,248}
\usepackage[pagebackref,breaklinks,colorlinks,allcolors=cvprblue]{hyperref}

\AtBeginDocument{%
  \crefname{table}{Table}{Tables}%
  \Crefname{table}{Table}{Tables}%
  \crefname{section}{Section}{Sections}%
  \Crefname{section}{Section}{Sections}%
  \crefname{subsection}{Section}{Sections}%
  \Crefname{subsection}{Section}{Sections}%
  \crefname{subsubsection}{Section}{Sections}%
  \Crefname{subsubsection}{Section}{Sections}%
}

\usepackage{amsthm}
\usepackage{mdframed}

\mdfdefinestyle{theoremstyle}{
  backgroundcolor=gray!5,
  linecolor=gray!30,
  linewidth=0.5pt,
  topline=true,
  bottomline=true,
  leftline=true,
  rightline=true,
  innertopmargin=8pt,
  innerbottommargin=8pt,
  innerleftmargin=8pt,
  innerrightmargin=8pt,
  skipabove=10pt,
  skipbelow=10pt,
  roundcorner=2pt
}

\theoremstyle{plain}
\newtheorem{proposition}{Proposition}[section]
\newtheorem{lemma}[proposition]{Lemma}

\makeatletter
\renewenvironment{proof}[1][\proofname]{%
  \par
  \pushQED{\qed}%
  \normalfont \topsep6\p@\@plus6\p@\relax
  \trivlist
  \item[\hskip\labelsep
        \itshape\bfseries
    #1\@addpunct{.}]\ignorespaces
}{%
  \popQED\endtrivlist\@endpefalse
}
\makeatother

\usepackage{algorithm}
\usepackage{algpseudocode}
\usepackage{multirow}
\usepackage{makecell}
\usepackage{colortbl}
\usepackage{array}
\usepackage{tabularx}
\usepackage{float}
\usepackage{placeins}

\newcolumntype{N}{>{\centering\arraybackslash}m{1.00cm}}

\newcommand{\dtGain}[2]{%
  \makebox[5.80em][c]{%
    \makebox[2.20em][c]{#1}%
    \rlap{\kern-0.1em\smash{\raisebox{-0.3ex}{%
      \makebox[1.75em][l]{\scriptsize\textcolor{green!60!black}{+#2}}%
    }}}%
  }%
}
\newcommand{\dtLoss}[2]{%
  \makebox[5.80em][c]{%
    \makebox[2.20em][c]{#1}%
    \rlap{\kern-0.1em\smash{\raisebox{-0.3ex}{%
      \makebox[1.75em][l]{\scriptsize\textcolor{red!60}{#2}}%
    }}}%
  }%
}
\AtBeginEnvironment{algorithm}{\ifdefined\nolinenumbers\nolinenumbers\fi}
\AtEndEnvironment{algorithm}{\ifdefined\linenumbers\linenumbers\fi}
\algrenewcommand\algorithmicrequire{\textbf{Input:}}
\algrenewcommand\algorithmicensure{\textbf{Output:}}
\algnewcommand\algorithmicbreak{\textbf{break}}
\algnewcommand\Break{\algorithmicbreak}

\def\confName{CVPR}
\def\confYear{2027}
\title{DT-LoRA: Reaching Beyond the Tangent Space of Low-Rank Adaptation}
\author{
  Yangming Zhang$^{1}$, Ling Yi$^{1}$\thanks{Corresponding authors.}, Datong Li$^{2}$,
  Jinliang Ding$^{3}$,\\
  Xiaojie Wang$^{1}$, Tiecheng Song$^{1}$, Zhaolong Ning$^{1}$\\[0.5em]
  $^{1}$Chongqing University of Posts and Telecommunications, China\\
  $^{2}$Anhui University of Science and Technology\\
  $^{3}$Northeastern University, China\\
}

\begin{document}
\maketitle

\begin{abstract}
Low-rank adaptation (LoRA) often underperforms full fine-tuning. We investigate this gap through a tangent-space lens: factorized updates are confined to the tangent space of the current low-rank factors, leaving gradient components outside this space unexploited. An analysis of ten representative LoRA variants reveals that these residual components dominate the gradient energy and exhibit strong spectral concentration.
Motivated by this finding, we propose Dynamic Tangent LoRA (DT-LoRA), which partitions a fixed rank budget into stable and dynamic branches to capture the principal residual gradient directions. 
We show that these directions define the optimal update of a local quadratic model under orthogonal-complement and rank constraints. Matrix-free randomized SVD recovers them without materializing the residual gradient, while controlled write-back preserves accumulated updates and stabilizes training within the allocated rank. 
DT-LoRA adds no trainable parameters or inference cost over LoRA. Experiments across computer vision and natural language processing tasks show consistent performance gains, demonstrating the effectiveness of DT-LoRA. Moreover, despite being a general-purpose adaptation method, DT-LoRA outperforms task-specific approaches to open-vocabulary domain generalization for semantic segmentation, achieving state-of-the-art results. Code is available at \url{https://github.com/Snipine/DT-LoRA}.

\end{abstract}

\section{Introduction}
\label{sec:intro}

\begin{figure}[t]
	\captionsetup{skip=2pt}
	\centering
	\includegraphics[width=\linewidth]{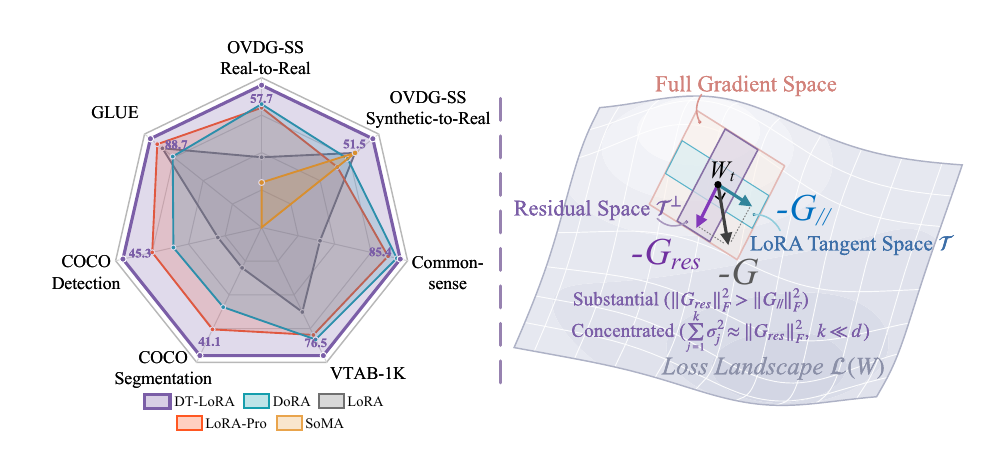}
\caption{\textbf{DT-LoRA reaches beyond the tangent space of LoRA.}
	\textbf{Left:} Performance comparison of DT-LoRA with representative LoRA variants across seven vision and language task settings.
	\textbf{Right:} Most current-step gradient energy lies in the spectrally concentrated residual $G_{\mathrm{res}}$, which is inaccessible to LoRA's current first-order factor update.}
	\label{fig:teaser}
\end{figure}

Large-scale pre-trained models serve as general-purpose foundations for a wide range of vision and language tasks~\cite{Vyas_2026_CVPR,NEURIPS2025_48b383b2}. However, as model sizes increase, full fine-tuning becomes prohibitively expensive, requiring all model parameters to be updated and a separate set of weights to be stored for each downstream task~\cite{pmlr-v235-gao24o}. Parameter-efficient fine-tuning (PEFT) reduces this cost by freezing the pre-trained weights and optimizing only a small set of parameters~\cite{jia2022vpt,Mai_2025_CVPR}. Among PEFT methods, low-rank adaptation (LoRA)~\cite{hu2022lora} has been widely adopted for its low-rank parameterization of weight updates, which can be merged into the pre-trained weights without additional inference overhead.

Nevertheless, LoRA often falls short of full fine-tuning~\cite{ICLR2025_382b95a6}. Prior work seeks to narrow this gap by modifying the update structure~\cite{si2025liera,ICLR2025_48c368f1,Li_2025_ICCV,Albert_2025_ICCV}, improving subspace selection, switching, and initialization~\cite{hao2024flora,zhao2024galore,Das_2025_ICCV,Park_2026_CVPR}, or correcting factor gradients and optimizer states to better approximate full fine-tuning updates~\cite{wang2025lorapro,yu2025altlora,tastan2026loft}. However, a fundamental local constraint remains for updates through the current LoRA factors: under factor perturbations $\delta A$ and $\delta B$, the first-order weight update lies entirely in the tangent space defined by the current factors. The dense gradient component orthogonal to this space therefore cannot be captured by any first-order factor update, regardless of how the factor gradients are corrected at the current step.

This raises a question: how much gradient information lies outside the current tangent space? To answer this question, we analyze the tangent-space residual gradient across ten representative LoRA variants and identify two consistent properties (see \cref{sec:empirical-gradient-analysis} for details). First, the residual carries substantial energy across all evaluated methods. Second, it exhibits strong spectral concentration. For instance, the residual accounts for over 85\% of total gradient energy, and the top two singular components capture over 90\% of the residual energy on OVDG-SS. These findings indicate that substantial gradient information lies outside the current tangent space, yet most of this residual information can be recovered through a very low-rank approximation.

A natural way to exploit the residual gradient is to continually incorporate its principal directions into the trainable subspace. However, increasing the rank incurs a proportional parameter cost, whereas replacing existing factors at a fixed rank may discard accumulated adaptation information and disrupt training continuity. The key challenge is therefore to incorporate these principal residual directions within a fixed rank budget while retaining learned information, and preserving mergeability into the pre-trained weights.

To address this challenge, we propose Dynamic Tangent LoRA (DT-LoRA), which partitions a fixed rank budget into a stable branch that accumulates task-specific directions as in standard LoRA and a dynamic branch that periodically absorbs principal residual directions outside the current tangent space. We prove that, under orthogonal-complement and rank constraints, the optimal update for a local quadratic model is given by the truncated SVD of the tangent-space residual gradient. Crucially, this residual inherits a factored structure from backpropagation, enabling its products with narrow matrices to be computed directly from cached input activations and output gradients. This structure allows us to construct matrix-free operators for Matrix-Free Randomized SVD, which extracts the principal residual directions via randomized SVD~\cite{halko2011finding} without explicitly forming the residual gradient. To incorporate these directions, we introduce the Controlled Write-Back mechanism, which merges them with the current dynamic branch update, compresses the result back to the original rank, and bounds the relative change at each refresh. After training, both branches are merged into the pre-trained weights, yielding the same inference cost as LoRA. We evaluate DT-LoRA across a broad range of vision and language tasks and observe consistent performance gains under matched trainable-parameter budgets. Notably, even as a general fine-tuning method without any task-specific design, DT-LoRA achieves state-of-the-art performance on the OVDG-SS benchmark.

Our main contributions are summarized as follows:
\begin{itemize}
\item We identify a geometric limitation of factorized LoRA: substantial gradient energy lies outside the current tangent space. Across ten representative LoRA variants, we show that this tangent-space residual exhibits strong spectral concentration, and establish that its truncated SVD yields the optimal update under orthogonal-complement and rank constraints for a local quadratic model.
\item We propose DT-LoRA, which partitions a fixed rank budget into Stable and Dynamic branches to incorporate principal residual directions. Matrix-Free Randomized SVD extracts these directions without materializing the residual gradient, while Controlled Write-Back integrates them with accumulated updates and bounds the change at each refresh, enabling dynamic tangent-space adaptation without additional trainable parameters or inference overhead.
\item We evaluate DT-LoRA across diverse vision and language tasks, with consistent performance gains under matched trainable-parameter budgets. On OVDG-SS, DT-LoRA outperforms task-specific methods and achieves state-of-the-art results.
\end{itemize}

\section{Related Work}
\label{sec:related-work}

\subsection{Parameter-Efficient Fine-Tuning}
To reduce the memory and storage costs of full fine-tuning, parameter-efficient fine-tuning (PEFT) freezes the pre-trained backbone and updates only a small set of parameters. Adapter-based methods~\cite{houlsby2019parameter,he2022unified,chen2022adaptformer,hu2023llm,luo2023repadapter} insert lightweight trainable modules into the backbone, while prompt-based methods~\cite{lester2021power,li2021prefix,jia2022vpt,razdaibiedina2023residual,Hao_2025_CVPR} learn continuous prompts at the input or intermediate layers. Selective methods~\cite{benzaken2022bitfit,lian2022ssf,Zhang_2025_CVPR,Zhao_2025_CVPR} tune a subset of existing parameters or introduce lightweight feature transformations. Among reparameterization-based methods, LoRA~\cite{hu2022lora} learns low-rank weight updates that can be merged into the pre-trained weights after training, avoiding additional inference overhead. These properties have motivated extensive research on LoRA and its variants~\cite{hayou2024loraplus,NEURIPS2023_1feb8787}.
\begin{figure*}[t]
	\captionsetup{skip=2pt}
	\centering
	\begin{subfigure}[t]{0.31\textwidth}
		\centering
		\includegraphics[width=\linewidth]{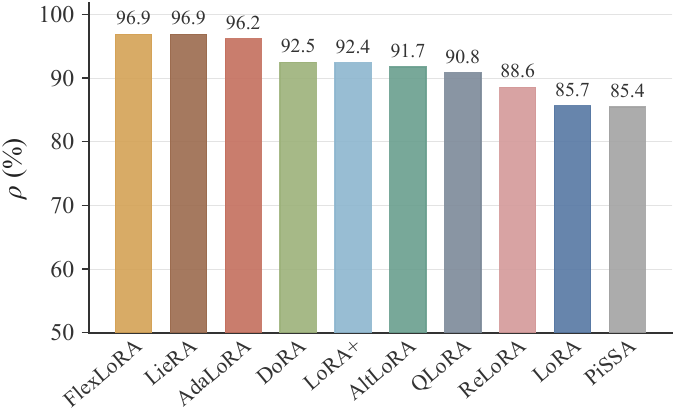}
		\caption{Residual-gradient energy ratio $\rho$.}
		\label{fig:residual-rho-10methods}
	\end{subfigure}
	\hfill
	\begin{subfigure}[t]{0.31\textwidth}
		\centering
		\includegraphics[width=\linewidth]{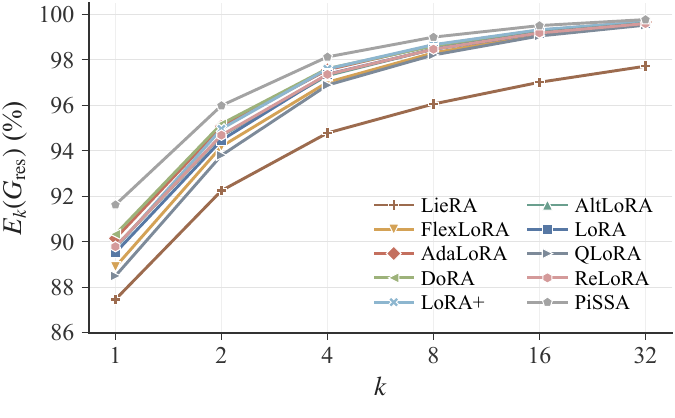}
		\caption{Top-$k$ residual spectral energy $E_k$.}
		\label{fig:residual-energy-curves-10methods}
	\end{subfigure}
	\hfill
	\begin{subfigure}[t]{0.31\textwidth}
		\centering
		\includegraphics[width=\linewidth]{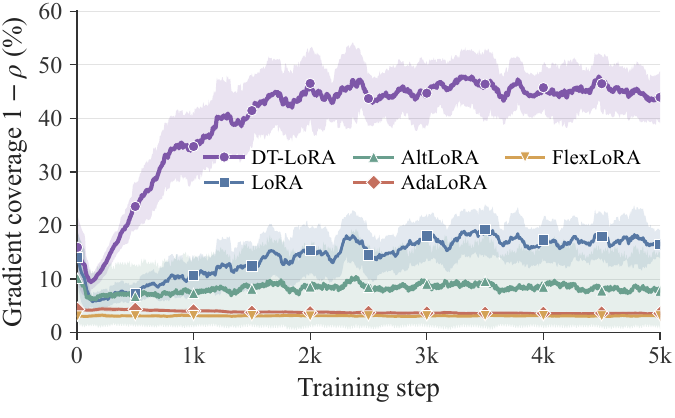}
		\caption{Gradient coverage during training.}
		\label{fig:gradient-coverage}
	\end{subfigure}
	\caption{Analysis of the tangent-space residual gradient. (a)~Residual-gradient energy ratio $\rho$ across ten LoRA variants. (b)~Cumulative residual spectral energy $E_k$. (c)~Gradient coverage $1-\rho$ during training. Across variants, the residual carries substantial energy concentrated in a few singular directions, while DT-LoRA's increasing coverage indicates progressive incorporation of these directions.}
	\label{fig:residual-energy-analysis}
\end{figure*}
\subsection{Low-Rank Adaptation and Its Variants}
Subsequent work improves LoRA mainly in three aspects. The first concerns the structure of the update and rank allocation~\cite{Chen_2025_CVPR,Fontana_2026_CVPR}: DoRA~\cite{liu2024dora} decomposes the weight into magnitude and direction and updates the direction in low-rank form; LieRA~\cite{si2025liera} extends low-rank updates to high-dimensional parameters such as convolution kernels via the exponential map of Lie groups; AdaLoRA~\cite{zhang2023adalora} and FlexLoRA~\cite{liu2026flexlora} dynamically allocate the rank budget across layers according to matrix importance. The second concerns subspace initialization, selection, and switching~\cite{zhang2025loraone,pmlr-v267-li25bm}: PiSSA~\cite{meng2024pissa} initializes the adapter with the principal singular components of the pre-trained weights, MiLoRA~\cite{wang2025milora} and SoMA~\cite{yun2025soma} instead select the minor singular components to reduce interference with pre-trained knowledge, and LoRA-GA~\cite{wang2024loraga} aligns the initial update direction with the gradient of full fine-tuning; ReLoRA~\cite{lialin2023relora} periodically merges low-rank updates into the base weights and reinitializes the adapters, while GaLore~\cite{zhao2024galore} periodically refreshes low-rank gradient projections for full-parameter learning, although their accumulated updates are not retained in a fixed-rank adapter. The third concerns optimization: LoRA-Pro~\cite{wang2025lorapro} and AltLoRA~\cite{yu2025altlora} correct the factor gradients to better approximate the full fine-tuning gradient in weight space, while LoFT~\cite{tastan2026loft} further aligns optimizer moment estimates. Nevertheless, the first-order weight updates of these methods remain confined to the tangent space induced by the current factors, without explicitly compensating for the omitted gradient components. GDLoRA~\cite{ouyang2026gdlora}, developed concurrently, also begins to explore these omitted gradient components, but applies the normal component directly to the base weights, so task-specific updates cannot be stored or switched using compact adapters alone.

\section{Methodology}
\label{sec:method}

\subsection{Preliminaries}
\label{sec:preliminaries}

We first review LoRA. Consider a linear projection layer in a pre-trained model, and denote its frozen pre-trained weight by $W_0\in\mathbb{R}^{d_{\mathrm{out}}\times d_{\mathrm{in}}}$. Instead of updating $W_0$ directly, standard LoRA adds a low-rank update to it:
\begin{equation}
	W=W_0+D,\qquad D=sBA,
	\label{eq:vanilla-lora}
\end{equation}
where $A\in\mathbb{R}^{r\times d_{\mathrm{in}}}$ and $B\in\mathbb{R}^{d_{\mathrm{out}}\times r}$ are trainable factors, $r$ is the factor width, and $s=\alpha/r$ is a scaling factor. Since $\operatorname{rank}(D)\leq r$, this parameterization represents a low-rank weight update with $r(d_{\mathrm{in}}+d_{\mathrm{out}})$ parameters.

\subsection{Tangent-Space Residual Analysis}
\label{sec:tangent-residual-lora}

\subsubsection{Tangent-Space Residual Gradient}
\label{sec:tangent-residual}
LoRA's low-rank parameterization constrains both the rank of the weight update and the gradient directions accessible at each optimization step. Under infinitesimal perturbations $\delta A$ and $\delta B$ of the factors in \cref{eq:vanilla-lora}, the first-order change is
\begin{equation}
	\delta D=s\bigl(B\,\delta A+\delta B\,A\bigr).
	\label{eq:lora-tangent}
\end{equation}
The set of all such changes defines the tangent space $\mathcal{T}$ induced by the current LoRA factors, \ie, the weight directions accessible to first order through factor updates.

At a probe step, let $X\in\mathbb{R}^{N\times d_{\mathrm{in}}}$ denote the layer's input activations and $R=\partial\mathcal{L}/\partial Y\in\mathbb{R}^{N\times d_{\mathrm{out}}}$ the output gradients from the same backward pass, where $Y=XW^{\top}$ and $\mathcal{L}$ is the training loss. We estimate the local weight gradient by the token-averaged outer product:
\begin{equation}
	G=\frac{1}{N}R^{\top}X.
	\label{eq:probe-gradient}
\end{equation}

Let $Q_A=\operatorname{orth}(A^{\top})$ and $Q_B=\operatorname{orth}(B)$ be orthonormal bases for the row space of $A$ and the column space of $B$, respectively. The orthogonal projectors onto their complementary subspaces are
\begin{equation}
	P_A^{\perp}=I-Q_AQ_A^{\top},\qquad
	P_B^{\perp}=I-Q_BQ_B^{\top}.
	\label{eq:orth-complements}
\end{equation}
We define the \emph{tangent-space residual gradient} as the component of the local gradient orthogonal to the current tangent space:
\begin{equation}
	G_{\mathrm{res}}
	=P_B^{\perp}GP_A^{\perp}
	=\frac{1}{N}(RP_B^{\perp})^{\top}(XP_A^{\perp}).
	\label{eq:tangent-residual}
\end{equation}
By construction, $Q_B^{\top}G_{\mathrm{res}}=0$ and $G_{\mathrm{res}}Q_A=0$. Thus, by \cref{eq:lora-tangent}, for every first-order change $\delta D\in\mathcal{T}$, we have
\begin{equation}
	\langle G_{\mathrm{res}},\delta D\rangle_F=0,
	\label{eq:residual-orthogonality}
\end{equation}
where $\langle\cdot,\cdot\rangle_F$ denotes the Frobenius inner product. Since $G-G_{\mathrm{res}}\in\mathcal{T}$, $G_{\mathrm{res}}$ is precisely the orthogonal projection of $G$ onto $\mathcal{T}^{\perp}$, \ie, the gradient component inaccessible to first-order updates of the current factors.

\subsubsection{Empirical Analysis}
\label{sec:empirical-gradient-analysis}

To quantify how well the current tangent space covers the gradient during training, we use CAT-Seg with an EVA02 ViT-L/14 backbone as a case study. We examine the query, key, value, and output projections in all layers of the vision and text encoders during fine-tuning. At each probe step, we compute the local gradient $G_i$ of the $i$-th matrix by \cref{eq:probe-gradient} and its residual $G_{\mathrm{res},i}$ by \cref{eq:tangent-residual}. For each LoRA variant with a rank budget of $r=16$, we define the residual with respect to the tangent space induced by its own parameterization. We measure the residual energy ratio and cumulative residual spectral energy as
\begin{equation}
	\rho = \frac{\sum_i\lVert G_{\mathrm{res},i}\rVert_F^2}{\sum_i\lVert G_i\rVert_F^2},
	\qquad
	E_k(G_{\mathrm{res}}) = \frac{\sum_i\sum_{j=1}^{k}\sigma_j^2(G_{\mathrm{res},i})}
	{\sum_i\lVert G_{\mathrm{res},i}\rVert_F^2},
	\label{eq:gradient-coverage-metrics}
\end{equation}
where $\sigma_j(\cdot)$ denotes the $j$-th singular value. Here $\rho$ is the fraction of gradient energy outside the current tangent space, so $1-\rho$ measures its coverage, and $E_k$ is the fraction of residual energy retained by the top-$k$ singular components of each matrix. Both metrics aggregate energies over all matrices at each probe step and are then averaged over probe steps.
\begin{figure*}[t]
	\captionsetup{skip=2pt}
	\centering
	\includegraphics[width=\textwidth]{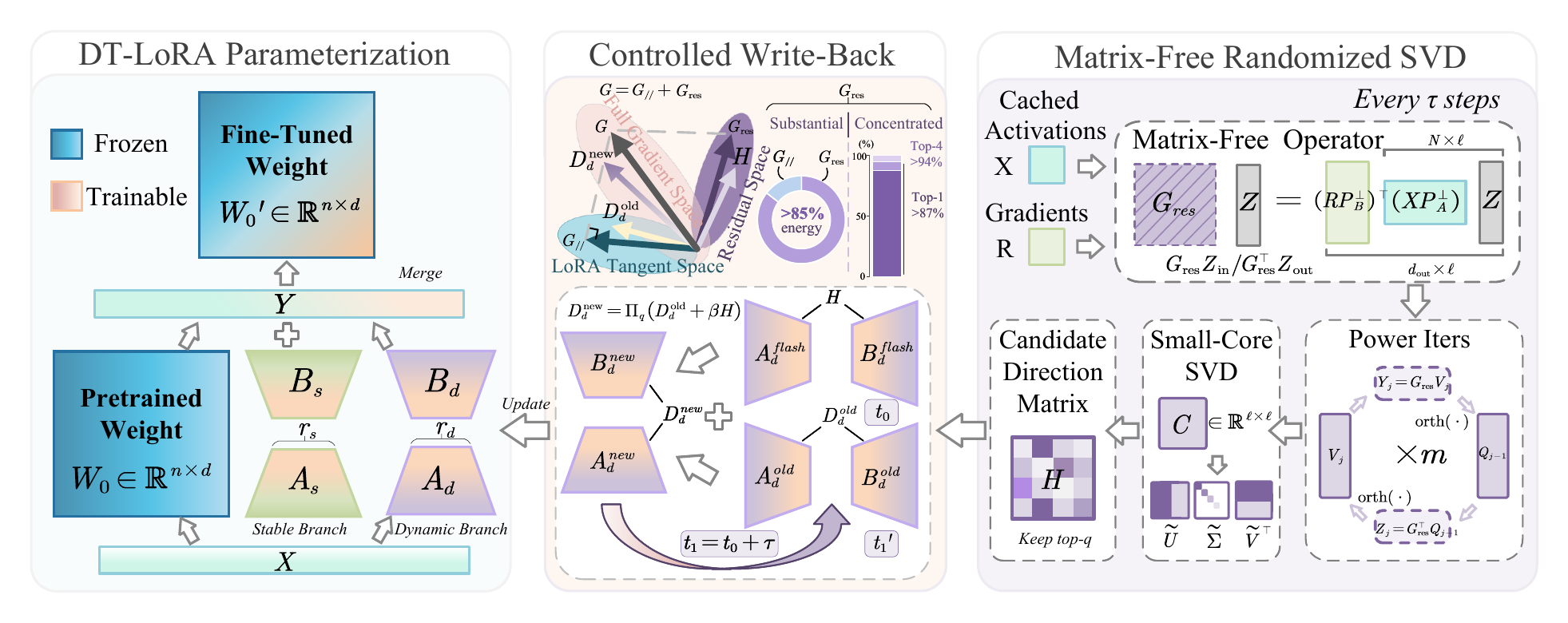}
	\caption{Overview of DT-LoRA. A fixed rank budget is split into a Stable branch and a Dynamic branch. Every $\tau$ steps, Matrix-Free Randomized SVD extracts the principal directions of the tangent-space residual gradient from the cached activations $X$ and output gradients $R$, yielding a candidate update $H$. Controlled Write-Back then merges $H$ into the Dynamic branch, keeping the branch at rank $q$. 
	}
	\label{fig:overview}
\end{figure*}

\cref{fig:residual-energy-analysis} reveals two consistent phenomena across ten LoRA variants. First, the residual accounts for a substantial fraction of gradient energy: $\rho$ exceeds 85\% for every method, ranging from 85.4\% for PiSSA to approximately 96.9\% for FlexLoRA and LieRA, with standard LoRA at 85.7\%. Second, the residual exhibits strong spectral concentration: $E_1$ exceeds 87\% and $E_4$ exceeds 94\% for all methods. The high residual energy ratio reveals substantial gradient information inaccessible to first-order updates, while its spectral concentration suggests that most of this information can be captured at a low rank. Similar trends hold for additional tasks, as detailed in Appendix~\cref{app:coco-residual-analysis}. Motivated by these findings, DT-LoRA tracks the dominant residual directions during training and incorporates them through a low-rank branch.

\subsection{Dynamic Tangent LoRA (DT-LoRA)}
\label{sec:dt-lora}

We propose DT-LoRA, which splits a fixed rank budget into a Stable branch that accumulates task-specific directions and a Dynamic branch that periodically absorbs the principal directions of the tangent-space residual gradient (\cref{fig:overview}). We first show that these directions yield the optimal low-rank update of a local quadratic model in the residual space. To realize this update, Matrix-Free Randomized SVD extracts them from cached activations $X$ and output gradients $R$ without forming $G_{\mathrm{res}}$, and Controlled Write-Back merges them into the Dynamic branch within its rank budget while bounding refresh-induced perturbations.

\subsubsection{Dual-Branch Parameterization}
\label{sec:stable-dynamic-parameterization}

For each adapted layer, we parameterize the weight update using two branches:
\begin{equation}
	W=W_0+D_s+D_d,
	\qquad
	D_s=sB_sA_s,\quad D_d=sB_dA_d,
	\label{eq:dt-lora-param}
\end{equation}
where $A_s,B_s$ and $A_d,B_d$ are the factors of the Stable and Dynamic branches, with widths $r_s$ and $r_d=q$ and a fixed total width $r=r_s+r_d$. Both branches are trained as standard LoRA, but only the Dynamic branch is refreshed. The Stable branch retains the learned directions, while the Dynamic branch reserves a fixed capacity for new ones. Since the tangent space is determined by both branches, we compute $G_{\mathrm{res}}$ by substituting the concatenated factors
\begin{equation}
	A=\begin{bmatrix}A_s\\A_d\end{bmatrix},
	\qquad
	B=\begin{bmatrix}B_s&B_d\end{bmatrix}.
	\label{eq:joint-factors}
\end{equation}

\subsubsection{Closed-Form Low-Rank Residual Update}
\label{sec:closed-form-update}

The strong spectral concentration of the tangent-space residual gradient motivates a low-rank update along its principal directions. We seek this update in the orthogonal complement of the current tangent space, $\mathcal{T}^{\perp}=\{\Delta:\Delta=P_B^{\perp}\Delta P_A^{\perp}\}$, which contains $G_{\mathrm{res}}$, and formulate the following optimization problem.

\begin{proposition}[Optimal low-rank update in the residual space]
	\label{prop:optimal-residual-update}
	Given a step size $\eta>0$ and a target rank $q$, consider the local quadratic model in the residual space $\mathcal{T}^{\perp}$:
	\begin{equation}
	    \min_{\substack{\operatorname{rank}(\Delta)\leq q\\
	    \Delta\in\mathcal{T}^{\perp}}}
	    \left\langle G_{\mathrm{res}},\Delta\right\rangle_F
	    +\frac{1}{2\eta}\|\Delta\|_F^2.
	    \label{eq:residual-quadratic}
	\end{equation}
	The optimal solution to this problem is
	\begin{equation}
	    \Delta^{\star}=-\eta[G_{\mathrm{res}}]_q,
	    \label{eq:ideal-direction}
	\end{equation}
	where $[G_{\mathrm{res}}]_q$ denotes the rank-$q$ truncated SVD of $G_{\mathrm{res}}$.
\end{proposition}

Completing the square reduces \cref{eq:residual-quadratic} to a rank-$q$ approximation of $-\eta G_{\mathrm{res}}$, which is solved by the Eckart--Young theorem~\cite{eckart1936approximation}. The solution remains feasible because the singular vectors of $G_{\mathrm{res}}\in\mathcal{T}^{\perp}$ lie in the ranges of $P_B^{\perp}$ and $P_A^{\perp}$. The full proof is given in Appendix~\cref{app:proof-optimal}.

\subsubsection{Matrix-Free Randomized SVD}
\label{sec:matrix-free-extraction}

\cref{eq:ideal-direction} requires the dominant singular directions of $G_{\mathrm{res}}$, but explicitly forming this large matrix and performing a dense SVD is computationally expensive. The key observation is that $G_{\mathrm{res}}$ never needs to be materialized: by \cref{eq:tangent-residual}, its product with any narrow matrix can be computed directly from the cached $X$ and $R$. This gives matrix-free operators that can be directly used by randomized SVD~\cite{halko2011finding} to approximate the required leading singular directions, yielding our Matrix-Free Randomized SVD.

\begin{lemma}[Matrix-free operators]
	\label{lem:matrix-free-operators}
	Given $X\in\mathbb{R}^{N\times d_{\mathrm{in}}}$, $R\in\mathbb{R}^{N\times d_{\mathrm{out}}}$, $P_A^{\perp}$, and $P_B^{\perp}$, for any $Z_{\mathrm{in}}\in\mathbb{R}^{d_{\mathrm{in}}\times \ell}$ and $Z_{\mathrm{out}}\in\mathbb{R}^{d_{\mathrm{out}}\times \ell}$,
	\begin{equation}
	\begin{aligned}
	    G_{\mathrm{res}}Z_{\mathrm{in}}
	    &=\frac{1}{N}(RP_B^{\perp})^{\top}
	      \bigl(XP_A^{\perp}Z_{\mathrm{in}}\bigr),\\
	    G_{\mathrm{res}}^{\top}Z_{\mathrm{out}}
	    &=\frac{1}{N}(XP_A^{\perp})^{\top}
	      \bigl(RP_B^{\perp}Z_{\mathrm{out}}\bigr),
	\end{aligned}
	    \label{eq:matrix-free-operators}
	\end{equation}
	each computed in $\mathcal{O}(N\ell d)$ time with $\mathcal{O}((N+d)\ell)$ extra memory, where $d=\max(d_{\mathrm{in}},d_{\mathrm{out}})$.
\end{lemma}

\cref{eq:matrix-free-operators} follows from expanding \cref{eq:tangent-residual}. Evaluated from right to left, each product involves only narrow intermediates of sizes $d\times\ell$ and $N\times\ell$. Let $\ell=q+p$ with oversampling parameter $p$. We draw a Gaussian random matrix $\Omega\in\mathbb{R}^{d_{\mathrm{in}}\times\ell}$ and form $Y_0=G_{\mathrm{res}}\Omega$ to obtain a randomized sketch of the dominant left singular subspace. To improve approximation accuracy, we perform $m$ power iterations, with intermediate orthogonalization for numerical stability:
\begin{equation}
	\begin{aligned}
		Q_{j-1} &= \operatorname{orth}(Y_{j-1}),\quad
		Z_j = G_{\mathrm{res}}^{\top}Q_{j-1},\\
		V_j &= \operatorname{orth}(Z_j),\quad
		Y_j = G_{\mathrm{res}}V_j,
		\qquad j=1,\ldots,m,
	\end{aligned}
	\label{eq:randomized-power-iteration}
\end{equation}
where $\operatorname{orth}(\cdot)$ returns an orthonormal basis via QR decomposition. We then construct
$Q_U=\operatorname{orth}(Y_m)$ and
$Q_V=\operatorname{orth}(G_{\mathrm{res}}^{\top}Q_U)$,
and form the core matrix
$C=Q_U^{\top}G_{\mathrm{res}}Q_V\in\mathbb{R}^{\ell\times\ell}$. After computing its SVD, $C=\widetilde U\widetilde\Sigma\widetilde V^{\top}$, we recover a rank-$q$ candidate descent direction:
\begin{equation}
    H=-Q_U\widetilde U_{[:,1:q]}\widetilde\Sigma_{[1:q,1:q]}\widetilde V_{[:,1:q]}^{\top}Q_V^{\top}
    \approx-[G_{\mathrm{res}}]_q.
    \label{eq:candidate-direction}
\end{equation}

All products involving $G_{\mathrm{res}}$ or $G_{\mathrm{res}}^{\top}$ are evaluated by \cref{eq:matrix-free-operators}, requiring $2m+3$ operator calls in total, so each refresh costs $\mathcal{O}(Nqd)$ for small constants $m$ and $p$. This procedure has three advantages: (i)~it never stores $G_{\mathrm{res}}$, which requires $\mathcal{O}(d_{\min}d)$ memory with $d_{\min}=\min(d_{\mathrm{in}},d_{\mathrm{out}})$; (ii)~it performs SVD only on the $\ell\times\ell$ core matrix; and (iii)~it reuses the cached $X$ and $R$ without additional forward or backward passes.

\begin{table*}[t]
	\captionsetup{skip=2pt}
	\centering
	\caption{VTAB-1K accuracy (\%) of fine-tuned ViT-B/16 pre-trained on ImageNet-21K. }
	\label{tab:vtab-results}
	\tiny
	\setlength{\tabcolsep}{1.6pt}
	\renewcommand{\arraystretch}{0.95}
	\resizebox{\textwidth}{!}{%
		\begin{tabular}{l c | c c c c c c c | c c c c | c c c c c c c c | c}
			\hline
			\multirow{11}{*}{Method}
			& \multirow{11}{*}{Params}
			& \multicolumn{7}{c|}{Natural}
			& \multicolumn{4}{c|}{Specialized}
			& \multicolumn{8}{c|}{Structured}
			& \multirow{11}{*}{Avg.} \\[0pt]
			
			& & \rotatebox{90}{CIFAR-100}
			& \rotatebox{90}{Caltech101}
			& \rotatebox{90}{DTD}
			& \rotatebox{90}{Flowers102}
			& \rotatebox{90}{Pets}
			& \rotatebox{90}{SVHN}
			& \rotatebox{90}{Sun397}
			& \rotatebox{90}{Camelyon}
			& \rotatebox{90}{EuroSAT}
			& \rotatebox{90}{Resisc45}
			& \rotatebox{90}{Retinopathy}
			& \rotatebox{90}{Clevr-Count}
			& \rotatebox{90}{Clevr-Dist}
			& \rotatebox{90}{DMLab}
			& \rotatebox{90}{KITTI-Dist}
			& \rotatebox{90}{dSpr-Loc}
			& \rotatebox{90}{dSpr-Ori}
			& \rotatebox{90}{sNORB-Azim}
			& \rotatebox{90}{sNORB-Ele} & \\
			\hline
			Full FT & 85.8M & 68.9 & 87.7 & 64.3 & 97.2 & 86.9 & 87.4 & 38.8 & 79.7 & 95.7 & 84.2 & 73.9 & 56.3 & 58.6 & 41.7 & 65.5 & 57.5 & 46.7 & 25.7 & 29.1 & 68.9 \\
			\hline
			Adapter \cite{houlsby2019parameter} & 0.16M & 69.2 & 90.1 & 68.0 & 98.8 & 89.9 & 82.8 & 54.3 & 84.0 & 94.9 & 81.9 & 75.5 & 80.9 & 65.3 & 48.6 & 78.3 & 74.8 & 48.5 & 29.9 & 41.6 & 73.9 \\
			LoRA \cite{hu2022lora} & 0.29M 
			& 67.1 & 91.4 & 69.4 & 98.8 & 90.4 & 85.3 & 54.0 
			& 84.9 & 95.3 & 84.4 & 73.6 
			& 82.9 & \textbf{69.2} & 49.8 & 78.5 & 75.7 & 47.1 & 31.0 & 44.0 
			& 74.6 \\
			AdaLoRA \cite{zhang2023adalora} & 0.29M & 70.5 & 90.3 & \underline{71.9} & \textbf{99.1} & \textbf{91.4} & 87.0 & \textbf{56.0} & 83.9 & 95.3 & 85.6 & 75.2 & 81.6 & \underline{68.1} & 51.8 & 80.2 & 75.3 & 49.1 & 28.9 & 41.1 & 75.1 \\
			LoRA-Pro \cite{wang2025lorapro} & 0.29M & 69.2 & 92.9 & 68.6 & 98.7 & 89.5 & \underline{89.2} & 54.8 & \textbf{87.4} & 95.4 & 85.0 & 74.6 & 79.3 & 64.4 & 52.3 & 79.0 & \underline{84.6} & 51.0 & \underline{36.3} & 39.6 & 75.6\\
			DoRA \cite{liu2024dora} & 0.31M & 70.0 & 91.9 & 71.3 & \underline{99.0} & 90.0 & 88.7 & \underline{55.5} & 85.4 & \underline{95.5} & \underline{86.8} & \textbf{76.5} & 82.5 & 66.5 & \underline{53.0} & \underline{80.3} & 82.5 & 46.8 & 30.1 & 41.9 & 75.8 \\   
			GDLoRA \cite{ouyang2026gdlora} & 0.29M & 70.8 & \underline{93.5} & 70.5 & \textbf{99.1} & 90.6 & 85.2 & \underline{55.5} & \underline{87.3} & 94.8 & 84.1 & 74.6 & \underline{83.6} & 63.2 & 50.4 & 79.6 & 84.3 & \underline{51.5} & 34.5 & \textbf{45.5} & 75.9 \\
			FlexLoRA \cite{liu2026flexlora} & 0.29M & \underline{71.2} & 92.5 & \textbf{72.1} & \textbf{99.1} & 90.7 & \textbf{90.6} & 54.7 & 86.4 & \textbf{96.0} & \textbf{87.5} & 76.0 & 81.9 & 65.6 & \textbf{53.9} & 80.2 & 80.6 & 45.6 & 29.9 & 41.1 & \underline{76.0} \\

			\rowcolor{dtlorapurple} DT-LoRA & 0.29M & \textbf{71.5} & \textbf{94.0} & 70.4 & \underline{99.0} & \underline{91.2} & 86.5 & 55.2 & \textbf{87.4} & 95.2 & 84.5 & \underline{76.1} & \textbf{83.7} & 63.4 & 51.8 & \textbf{82.8} & \textbf{84.9} & \textbf{51.7} & \textbf{37.0} & \underline{44.3} & \textbf{76.5} \\
			
			\hline
		\end{tabular}%
	}
\end{table*}
\begin{table*}[t]
	\captionsetup{skip=2pt}
	\caption{Object detection and instance segmentation results of fine-tuned ConvNeXt-V2-B on MS COCO.}
	\label{tab:coco-detection}
	\centering
	\fontsize{7.5pt}{8pt}\selectfont
	\setlength{\tabcolsep}{3pt}
	\renewcommand{\arraystretch}{1.0}
	
	\begin{tabularx}{\textwidth}{
			l@{\hspace{6pt}}c|
			*{7}{>{\centering\arraybackslash}X}|
			*{7}{>{\centering\arraybackslash}X}
		}
		\hline
		& &
		\multicolumn{7}{c|}{Detection}
		& \multicolumn{7}{c}{Instance Segmentation} \\[-2pt]
		
		Method & Params
		& mAP & AP$_{50}$ & AP$_{75}$ & AP$_S$ & AP$_M$ & AP$_L$ & Avg.
		& mAP & AP$_{50}$ & AP$_{75}$ & AP$_S$ & AP$_M$ & AP$_L$ & Avg. \\
		\hline
		Full FT & 104.97M
		& 49.0 & 69.1 & 54.6 & 31.7 & 53.5 & 61.7 & 53.3
		& 43.4 & 66.2 & 47.6 & 23.4 & 47.1 & 60.9 & 48.1 \\
		\hline
		
		LoRA$_{r=16}$ \cite{hu2022lora} & 17.27M
		& 35.5 & 55.9 & 38.9 & 22.7 & 38.4 & 45.1 & 39.4
		& 33.6 & 53.1 & 36.3 & 17.4 & 35.9 & 48.2 & 37.4 \\
		
		LieRA \cite{si2025liera} & 17.27M
		& 39.1 & 59.9 & 43.1 & \underline{25.9} & 42.6 & 49.2 & 43.3
		& 37.0 & 57.6 & \underline{40.2} & \underline{20.1} & 40.0 & 53.2 & 41.4 \\
		
		DoRA \cite{liu2024dora} & 17.27M
		& \underline{39.7} & \underline{60.6} & \underline{43.6}
		& 25.3 & \underline{43.2} & \underline{51.4} & \underline{44.0}
		& \underline{37.2} & \underline{57.9} & 40.1 & 19.4
		& \underline{40.2} & \underline{53.9} & \underline{41.5} \\
		
		\rowcolor{dtlorapurple}
		DT-LoRA & 17.27M
		& \textbf{44.1} & \textbf{65.2} & \textbf{48.2}
		& \textbf{28.7} & \textbf{47.9} & \textbf{56.5} & \textbf{48.4}
		& \textbf{40.3} & \textbf{62.6} & \textbf{43.6}
		& \textbf{21.9} & \textbf{43.5} & \textbf{57.1} & \textbf{44.8} \\
		\hline
		
		LoRA$_{r=32}$ \cite{hu2022lora} & 34.54M
		& 35.9 & 56.4 & 39.7 & 22.9 & 39.6 & 45.7 & 40.0
		& 34.4 & 54.2 & 37.1 & 18.3 & 36.7 & 48.9 & 38.3 \\
		
		DoRA \cite{liu2024dora} & 34.54M
		& 40.3 & 61.1 & 44.2 & 24.1 & 44.0 & 52.2 & 44.3
		& 37.4 & 58.4 & 40.6 & 18.5 & 40.2 & 53.6 & 41.5 \\
		
		LieRA \cite{si2025liera} & 34.54M
		& 40.5 & 61.2 & 45.1 & 25.3 & 44.1 & 50.8 & 44.5
		& 38.2 & 58.9 & 41.8
		& \underline{19.8} & 41.2 & 53.9 & 42.3 \\
		
		LoRA-Pro \cite{wang2025lorapro}& 34.54M
		& \underline{42.4} & \underline{63.3} & \underline{46.5}
		& \underline{26.1} & \underline{46.4} & \underline{54.8}
		& \underline{46.6}
		& \underline{39.1} & \underline{60.7} & \underline{42.3}
		& \underline{19.8} & \underline{42.7} & \underline{56.5}
		& \underline{43.5} \\
		
		\rowcolor{dtlorapurple}
		DT-LoRA & 34.54M
		& \textbf{45.3} & \textbf{66.0} & \textbf{49.7}
		& \textbf{28.9} & \textbf{49.3} & \textbf{57.7} & \textbf{49.5}
		& \textbf{41.1} & \textbf{63.5} & \textbf{44.5}
		& \textbf{22.5} & \textbf{44.7} & \textbf{58.3} & \textbf{45.8} \\
		\hline
		
	\end{tabularx}
\end{table*}

\subsubsection{Controlled Write-Back}
\label{sec:dynamic-refresh}
Directly replacing the Dynamic branch with candidate direction matrix $H$ would discard accumulated adaptation information and could introduce abrupt changes in the model weights. We therefore introduce Controlled Write-Back, which integrates $H$ through adaptive injection scaling, merge-and-compress, and a relative-change constraint.

Let $D_d^{\mathrm{old}}$ be the Dynamic branch update immediately before a refresh, and let $D_d^{\mathrm{anchor}}$ denote its state after the last scheduled refresh. We measure the accumulated change since that refresh as
\begin{equation}
	w_d =
	\left\|D_d^{\mathrm{old}}-D_d^{\mathrm{anchor}}\right\|_F,
\end{equation}
and scale $H$ accordingly:
\begin{equation}
    \beta=\frac{\rho_{\mathrm{refresh}}w_d}{\|H\|_F+\varepsilon},
    \label{eq:refresh-scale}
\end{equation}
where $\rho_{\mathrm{refresh}}>0$ controls the refresh ratio and $\varepsilon>0$ ensures numerical stability. Since $H\approx-[G_{\mathrm{res}}]_q$, the scaled candidate $\beta H$ approximates the optimal residual-space update in \cref{eq:ideal-direction}, with $\beta$ serving as a layer-specific step size calibrated to the recent change in the Dynamic branch.

We then combine $\beta H$ with the accumulated update and compress the result to the allocated rank $q$:
\begin{equation}
    D_d^{\mathrm{new}}
    =\Pi_q\bigl(D_d^{\mathrm{old}}+\beta H\bigr),
    \label{eq:dynamic-compress}
\end{equation}
where $\Pi_q$ denotes a best rank-at-most-$q$ approximation in the Frobenius norm. Since $D_d^{\mathrm{old}}$ and $H$ both have rank at most $q$, this compression can be performed via QR decompositions of their concatenated factors followed by an SVD of a core matrix of size at most $2q\times2q$.

To control the perturbation, we impose the jump constraint
\begin{equation}
	E_{\mathrm{actual}}=D_d^{\mathrm{new}}-D_d^{\mathrm{old}},
	\qquad
	\frac{\|E_{\mathrm{actual}}\|_F}
	{\|D_d^{\mathrm{old}}\|_F+\varepsilon}
	\leq\rho_{\mathrm{jump}},
	\label{eq:jump-gate}
\end{equation}
where $\rho_{\mathrm{jump}}>0$ limits the relative change per refresh. If violated, we reduce $\beta$ via backtracking and skip the refresh upon backtracking failure. Otherwise, we refactorize the accepted update as $D_d^{\mathrm{new}}=sB_dA_d$ and set $D_d^{\mathrm{anchor}}=D_d^{\mathrm{new}}$. Under the smoothness assumptions in Appendix~\cref{app:writeback}, this bounds the loss and gradient changes.

After training, both branches are merged into the pre-trained weights, so inference introduces no extra cost and is equivalent to standard LoRA.

\section{Experiments}
\label{sec:experiments}

We evaluate DT-LoRA on vision tasks spanning image classification, object detection, instance segmentation, and OVDG-SS, as well as language tasks covering commonsense reasoning and natural language understanding. 

\subsection{Experimental Setup}
\label{sec:exp-setup}

\noindent\textbf{Datasets, Models, and Evaluation Protocols.} For image classification, we fine-tune ImageNet-21K-pre-trained ViT-B/16~\cite{dosovitsky2021vit} on VTAB-1K~\cite{zhai2019large}, selecting hyperparameters using the 800/200 split and retraining on all 1,000 examples~\cite{jia2022vpt,xin2024vpetl}. Following LieRA~\cite{si2025liera}, we evaluate ConvNeXt-V2-B with Mask R-CNN~\cite{he2017maskrcnn} on MS COCO~\cite{lin2014coco}. For OVDG-SS, we follow S$^{2}$-Corr~\cite{zhao2026s2corr}, evaluating CAT-Seg~\cite{cho2024cat} with EVA02-CLIP~\cite{eva02} ViT-B/16 and ViT-L/14, and additionally PCA-Seg~\cite{yin2026pcaseg} with ViT-L/14. Models are trained on CS-7 or GTA-7 and evaluated on Dv-19 and Dv-58 using mIoU. For language tasks, we fine-tune LLaMA2-7B and LLaMA3-8B on Commonsense170K~\cite{hu2023llm}, and DeBERTaV3-base~\cite{he2023debertav3} on GLUE~\cite{wang2019glue}.

\begin{table*}[t]
	\captionsetup{skip=2pt}
	\caption{Comparison of different methods with EVA02 ViT-L/14 under the Real-to-Real OVDG-SS setting.}
	\label{tab:cs7-results}
	\centering
	\setlength{\tabcolsep}{0pt}
	\renewcommand{\arraystretch}{1.0}
	\resizebox{\textwidth}{!}{
		\begin{tabular}{
				@{}
				l<{\hspace{4pt}} @{}
				>{\hspace{4pt}}c<{\hspace{4pt}} @{}
				>{\hspace{4pt}}c<{\hspace{4pt}} @{}
				>{\hspace{4pt}}c<{\hspace{8pt}} @{} |
				*{4}{c} |
				*{5}{c}
				@{}
			}
			\hline
			\multirow{2}{*}{Method} & \multirow{2}{*}{Backbone} & \multirow{2}{*}{Params} & \multirow{2}{*}{\makecell{Training\\Data}} & \multicolumn{4}{c|}{Dv-19} & \multicolumn{5}{c}{Dv-58} \\
			& & & & \multicolumn{1}{c}{ACDC-19} & \multicolumn{1}{c}{BDD-19} & \multicolumn{1}{c}{Mapi-19} & \multicolumn{1}{c|}{Avg.} & \multicolumn{1}{c}{ACDC-41} & \multicolumn{1}{c}{BDD-41} & \multicolumn{1}{c}{Mapi-30} & \multicolumn{1}{c}{RW-10} & \multicolumn{1}{c}{Avg.} \\
			\hline
			CAT-Seg \cite{cho2024cat} & \multirow{12}{*}{\makecell{EVA02 \\ ViT-L/14}} & 71.61M & CS-7 & 48.5 & 48.5 & 50.8 & 49.3 & 59.2 & 60.0 & 41.1 & 39.5 & 50.0 \\
			PCA-Seg \cite{yin2026pcaseg} & & 71.61M & CS-7 & 53.8 & 53.7 & 57.5 & 55.0 & 60.9 & 60.3 & 50.3 & 39.0 & 52.6 \\
			DGInStyle+CAT-Seg \cite{jia2024dginstyle} & & 71.61M & CS-7 & 50.4 & 51.7 & 50.4 & 50.8 & 60.2 & 60.8 & 43.5 & 38.7 & 50.8 \\
			SoMA+CAT-Seg \cite{yun2025soma} & & 4.33M & CS-7 & 48.5 & 48.7 & 50.7 & 49.3 & 52.3 & 56.6 & 41.3 & 40.0 & 47.5 \\
			LieRA\smash{$_{\scalebox{0.7}{\text{CAT-Seg}}}$} \cite{si2025liera} & & 4.33M & CS-7 & 40.7 & 43.8 & 43.4 & 42.6 & 42.1 & 46.9 & 34.2 & 38.1 & 40.3 \\
			FlexLoRA\smash{$_{\scalebox{0.7}{\text{CAT-Seg}}}$} \cite{liu2026flexlora} & & 4.33M & CS-7 & 44.6 & 47.4 & 50.1 & 47.4& 42.5 & 45.7 & 41.2 & 43.9 & 43.3 \\
			LoRA \smash{$_{\scalebox{0.7}{\text{CAT-Seg}}}$} \cite{hu2022lora} & & 4.33M & CS-7 & 51.2 & 50.9 & 55.7 & 52.6 & 53.6 & 54.1 & 44.2 & 43.9 & 49.0 \\
			LoRA-Pro\smash{$_{\scalebox{0.7}{\text{CAT-Seg}}}$} \cite{wang2025lorapro} & & 4.33M & CS-7 & \underline{57.9} & 55.4 & 58.5 & 57.3 & \underline{63.5} & 61.4 & 49.8 & 40.3 & 53.7 \\
			GDLoRA\smash{$_{\scalebox{0.7}{\text{CAT-Seg}}}$}  \cite{ouyang2026gdlora}  & & 4.33M %
			 & CS-7 & 56.2 & 54.7	& 57.7	& 56.2	& 61.9	& \underline{62.7}	& 45.4	& \underline{46.4}	& 54.1 \\
            DoRA \smash{$_{\scalebox{0.7}{\text{CAT-Seg}}}$} \cite{liu2024dora} & & 4.47M & CS-7 & 55.7 & 55.6 & 59.9 & 57.1 & 62.3 & 61.5 & 50.3 & 44.1 & 54.6 \\

			S\smash{$^{\scalebox{0.7}{2}}$}-Corr \cite{zhao2026s2corr} & & 71.61M & CS-7 & 54.3 & 53.1 & 60.0 & 55.8 & 62.0 & 61.7 & 47.4 & 41.9 & 53.2 \\
			\rowcolor{dtlorapurple}
			DT-LoRA\smash{$_{\scalebox{0.7}{\text{PCA-Seg}}}$}  & & 4.33M & CS-7 & \dtGain{55.8}{2.0} & \dtGain{\underline{55.8}}{2.1} & \dtGain{\underline{60.5}}{3.0} & \dtGain{\underline{57.4}}{2.4} & \dtGain{61.0}{0.1} & \dtGain{61.1}{0.8} & \dtGain{\underline{50.6}}{0.3} & \dtGain{\textbf{47.9}}{8.9} & \dtGain{\underline{55.2}}{2.6} \\
			\rowcolor{dtlorapurple}  DT-LoRA\smash{$_{\scalebox{0.7}{\text{CAT-Seg}}}$} & & 4.33M & CS-7 & \dtGain{\textbf{58.9}}{10.4} & \dtGain{\textbf{57.3}}{8.8} & \dtGain{\textbf{61.5}}{10.7} & \dtGain{\textbf{59.2}}{9.9} & \dtGain{\textbf{63.8}}{4.6} & \dtGain{\textbf{64.3}}{4.3} & \dtGain{\textbf{51.7}}{10.6} & \dtGain{44.5}{5.0} & \dtGain{\textbf{56.1}}{6.1} \\
			
			\hline
		\end{tabular}
	}
\end{table*}

\begin{table*}[t]
	\captionsetup{skip=2pt}
	\caption{Comparison of different methods with EVA02 ViT-L/14 under the Synthetic-to-Real OVDG-SS setting. }
	\label{tab:gta7-results}
	\centering
	\setlength{\tabcolsep}{0pt}
	\renewcommand{\arraystretch}{1.0}
	\resizebox{\textwidth}{!}{
		\begin{tabular}{
				@{}
				l<{\hspace{4pt}} @{}
				>{\hspace{4pt}}c<{\hspace{4pt}} @{}
				>{\hspace{4pt}}c<{\hspace{4pt}} @{}
				>{\hspace{4pt}}c<{\hspace{8pt}} @{} |
				*{5}{c} |
				*{5}{c}
				@{}
			}
			\hline
			\multirow{2}{*}{Method} & \multirow{2}{*}{Backbone} & \multirow{2}{*}{Params} & \multirow{2}{*}{\makecell{Training\\Data}} & \multicolumn{5}{c|}{Dv-19} & \multicolumn{5}{c}{Dv-58} \\
			& & & & \multicolumn{1}{c}{CS-19} & \multicolumn{1}{c}{ACDC-19} & \multicolumn{1}{c}{BDD-19} & \multicolumn{1}{c}{Mapi-19} & \multicolumn{1}{c|}{Avg.} & \multicolumn{1}{c}{ACDC-41} & \multicolumn{1}{c}{BDD-41} & \multicolumn{1}{c}{Mapi-30} & \multicolumn{1}{c}{RW-10} & \multicolumn{1}{c}{Avg.} \\
			\hline
			PCA-Seg \cite{yin2026pcaseg} &\multirow{12}{*}{\makecell{EVA02 \\ ViT-L/14}} & 71.61M & GTA-7 & 50.8 & \textbf{50.7} & 50.3 & 52.8 & 51.2 & 44.4 & 44.5 & 40.8 & \textbf{43.0} & 43.2 \\
			CAT-Seg \cite{cho2024cat} & & 71.61M & GTA-7 & 48.7 & 42.4 & 47.9 & 50.8 & 47.5 & 53.6 & 59.7 & 41.1 & 38.2 & 48.2 \\
			DGInStyle+CAT-Seg \cite{jia2024dginstyle} & & 71.61M & GTA-7 & 46.0 & 43.9 & 48.9 & 51.0 & 47.5 & 53.7 & 59.0 & 37.9 & 38.0 & 47.2 \\
			SoMA+CAT-Seg \cite{yun2025soma} & & 4.33M & GTA-7 & 52.4 & 47.3 & 51.8 & 53.7 & 51.3 & 57.6 & 60.1 & 42.6 & \underline{41.1} & 50.4 \\
			LieRA\smash{$_{\scalebox{0.7}{\text{CAT-Seg}}}$} \cite{si2025liera} & & 4.33M & GTA-7 & 43.7 & 39.4 & 41.0 & 43.2 & 41.8 & 35.8 & 37.5 & 37.3 & 37.8 & 37.1\\

			FlexLoRA\smash{$_{\scalebox{0.7}{\text{CAT-Seg}}}$} \cite{liu2026flexlora} & & 4.33M & GTA-7 & 46.2 & 40.4 & 45.4 & 47.9 & 45.0 & 39.0 & 44.7 & 38.6 & 36.3&  39.7 \\
			LoRA \smash{$_{\scalebox{0.7}{\text{CAT-Seg}}}$} \cite{hu2022lora} & & 4.33M & GTA-7 & 52.0 & 47.7 & 52.0 & 53.3 & 51.3 & 57.3 & \underline{61.3} & 43.4 & 39.6  & 50.4 \\
			LoRA-Pro\smash{$_{\scalebox{0.7}{\text{CAT-Seg}}}$} \cite{wang2025lorapro} & & 4.33M & GTA-7 & 50.9 & \underline{48.0} & 51.6 & 52.9 & 50.9 & 56.9 & 60.3& 42.8 & 38.2 & 49.6 \\
			GDLoRA\smash{$_{\scalebox{0.7}{\text{CAT-Seg}}}$}  \cite{ouyang2026gdlora} 	& &  4.33M %
			& GTA-7		&50.4	& 47.1	& 51.2	& 53.8	& 50.6	& \underline{58.5}	& 61.2	& 40.0	& 39.5	& 49.8 \\
			DoRA\smash{$_{\scalebox{0.7}{\text{CAT-Seg}}}$} \cite{liu2024dora} & & 4.47M & GTA-7 & \underline{52.6}  & 47.8 & 51.7 & 53.6   & 51.4  & 57.4 & 60.8 & \textbf{44.8} & 36.2  & 49.8 \\
			S\smash{$^{\scalebox{0.7}{2}}$}-Corr \cite{zhao2026s2corr} & & 71.61M & GTA-7 & 52.2 & 44.2 & 50.3 & 53.3 & 49.9 & 55.4 & 59.8 & 41.5 & 40.8 & 49.4 \\
			\rowcolor{dtlorapurple}  DT-LoRA\smash{$_{\scalebox{0.7}{\text{PCA-Seg}}}$} & & 4.33M & GTA-7 & \dtGain{\underline{52.6}}{1.8} & \dtLoss{46.9}{-3.8} & \dtGain{\textbf{52.7}}{2.4} & \dtGain{\underline{54.0}}{1.2} & \dtGain{\underline{51.6}}{0.4} & \dtGain{\textbf{59.4}}{15.0} & \dtGain{\textbf{62.6}}{18.1} & \dtGain{44.2}{3.4} & \dtLoss{40.9}{-2.1} & \dtGain{\textbf{51.8}}{8.6} \\
			\rowcolor{dtlorapurple}  DT-LoRA\smash{$_{\scalebox{0.7}{\text{CAT-Seg}}}$} & & 4.33M & GTA-7 & \dtGain{\textbf{52.8}}{4.1} & \dtGain{47.9}{5.5} & \dtGain{\underline{52.6}}{4.7} & \dtGain{\textbf{55.4}}{4.6} & \dtGain{\textbf{52.2}}{4.7} & \dtGain{58.1}{4.5} & \dtGain{61.0}{1.3} & \dtGain{\underline{44.5}}{3.4} & \dtGain{39.4}{1.2} & \dtGain{\underline{50.7}}{2.5} \\
			\hline
		\end{tabular}
	}
\end{table*}

\noindent\textbf{Implementation Details.} We compare DT-LoRA with representative general-purpose LoRA variants, including DoRA, LoRA-Pro, and FlexLoRA. For OVDG-SS, we additionally compare with ``domain generalization + open-vocabulary'' methods and the task-specific state-of-the-art S$^{2}$-Corr~\cite{zhao2026s2corr}. All experiments use PyTorch on NVIDIA A100 GPUs.
 For fair comparison, the competing LoRA variants are given parameter budgets matched to or larger than ours in all tasks, as reported in the tables.
 Except for VTAB-1K hyperparameter tuning, we follow prior recommended settings. Full training configurations are provided in Appendix~\cref{app:hyperparameters}.

\subsection{Results on Computer Vision Tasks}
\label{sec:vision-results}

\noindent\textbf{Image Classification.} As shown in \cref{tab:vtab-results}, DT-LoRA achieves an average accuracy of 76.5\% with 0.29M trainable parameters, surpassing LoRA (74.6\%), DoRA (75.8\%), and FlexLoRA (76.0\%) under matched parameter budgets. The largest gains occur on structured tasks, with improvements of 9.2 and 6.0 percentage points over LoRA on dSpr-Loc and sNORB-Azim, respectively. These results suggest that incorporating residual gradient directions is particularly beneficial for tasks that differ substantially from the natural-image pre-training distribution. DT-LoRA also improves performance on natural tasks, including a 4.4-point gain on CIFAR-100.

\noindent\textbf{Object Detection and Instance Segmentation.}
As shown in \cref{tab:coco-detection}, DT-LoRA outperforms all PEFT baselines on every metric at both $r=16$ and $r=32$. It reaches detection and segmentation mAP of 44.1 and 40.3, exceeding LoRA by 8.6 and 6.7 points and surpassing DoRA$_{r=32}$ and LieRA$_{r=32}$ despite using half the parameters. Doubling LoRA's rank yields marginal mAP gains of only 0.4 and 0.8 points, suggesting limited benefits from increasing rank alone in this setting. In contrast, DT-LoRA achieves substantially larger gains within a fixed rank budget by incorporating principal residual directions outside the current tangent space.

\noindent\textbf{Open-Vocabulary Domain Generalization in Semantic Segmentation.} \cref{tab:cs7-results,tab:gta7-results} show that DT-LoRA establishes new state-of-the-art results across segmentation frameworks. With EVA02 ViT-L/14, DT-LoRA$_{\text{CAT-Seg}}$ reaches 59.2 and 56.1 mIoU on Dv-19 and Dv-58 in real-to-real transfer, surpassing S$^{2}$-Corr by 3.4 and 2.9 points, respectively, and outperforming all compared LoRA variants. Notably, FlexLoRA and LieRA perform worse, possibly because their residual-gradient energy remains large but is not utilized in time, as observed in \cref{fig:residual-rho-10methods}. With PCA-Seg, DT-LoRA$_{\text{PCA-Seg}}$ achieves the highest reported Dv-58 mIoU of 51.8 in synthetic-to-real transfer, an 8.6-point improvement over PCA-Seg. With ViT-B/16, DT-LoRA also ranks first in real-to-real transfer (full results on ViT-B/16 are reported in Appendix~\cref{tab:cs7-full,tab:gta7-full}). These results demonstrate consistent improvements across backbones, segmentation frameworks, and source datasets without task-specific architectural modifications. 

\begin{table*}[t]
	\captionsetup{skip=2pt}
	\centering
	\caption{Accuracy (\%) on eight Commonsense170K evaluation benchmarks.}
	\label{tab:commonsense-results}
	
	\fontsize{7pt}{7pt}\selectfont
	\renewcommand{\arraystretch}{1.0}
	\setlength{\tabcolsep}{3pt}
	
	\begin{tabularx}{\textwidth}{llc*{9}{>{\centering\arraybackslash}X}}
		\hline
		Model & Method & Params & BoolQ & PIQA & SIQA
		& HellaS. & WinoG. & ARC-e & ARC-c & OBQA & Avg. \\
		\hline
		
		ChatGPT & - & --
		& 73.1 & 85.4 & 68.5 & 78.5 & 66.1
		& 89.8 & 79.9 & 74.8 & 77.0 \\
		\hline
		
		\multirow{6.5}{*}{LLaMA3-8B}
		& PiSSA$_{r=32}$ \cite{meng2024pissa}
		& 0.70\%
		& 67.1 & 81.1 & 77.2 & 83.6 & 78.9
		& 77.7 & 63.2 & 74.6 & 75.4 \\
		
		& LoRA$_{r=32}$ \cite{hu2022lora}
		& 0.70\%
		& 70.8 & 85.2 & 79.9 & 91.7 & 84.3
		& 84.2 & 71.2 & 79.0 & 80.8 \\
		
		& AdaLoRA$_{r=32}$ \cite{zhang2023adalora}
		& 0.70\%
		& 71.3 & 88.7 & \underline{80.1} & 94.5 & \textbf{86.2}
		& 90.2 & 78.8 & 85.8 & 84.5 \\
		
		& LoRA-Pro$_{r=32}$ \cite{wang2025lorapro}
		& 0.70\%
		& \underline{73.9} & \underline{89.0} & 78.6 & 94.2 & 85.9
		& \underline{90.8} & 78.3 & \underline{86.7} & 84.7 \\
		
		& DoRA$_{r=32}$ \cite{liu2024dora}
		& 0.70\%
		& \textbf{74.5} & \textbf{89.5} & 80.0
		& \underline{95.4} & \underline{86.0}
		& 90.6 & \textbf{79.9} & 85.6 & \underline{85.2} \\
		
		\rowcolor{dtlorapurple}
		& DT-LoRA$_{r=16}$
		& 0.35\%
		& \textbf{74.5} & 88.6 & \textbf{80.9} & \textbf{95.7} & 85.9
		& \textbf{91.1} & \underline{79.3} & \textbf{87.0} & \textbf{85.4} \\
		
		\hline
	\end{tabularx}
\end{table*}

\begin{table*}[t]
	\captionsetup{skip=2pt}
	\centering
	\caption{Results of fine-tuned DeBERTaV3-base on the GLUE development set, averaged over three runs with different random seeds.}
	\label{tab:glue-results}
	\resizebox{1\textwidth}{!}{%
		\footnotesize
		\renewcommand{\arraystretch}{1.02}
		\begin{tabular}{lc*{9}{c}}
			\hline
			Method & Params & CoLA (Mcc.) & MNLI (Acc.) & MRPC (Acc.) & RTE (Acc.) & QNLI (Acc.) & SST-2 (Acc.) & STS-B (Corr.) & QQP (Acc.) & Avg. \\
			\hline
			Full FT & 184.3M & 69.2 & 89.9 & 90.2 & 83.8 & 94.0 & 95.6 & 91.6 & 92.4 & 88.3 \\
			\hline
			DoRA$_{r=8}$~\cite{liu2024dora} & 1.41M & 67.5\rlap{$_{\pm 1.8}$} & 88.6\rlap{$_{\pm 1.3}$} & 90.0\rlap{$_{\pm 0.7}$} & 83.5\rlap{$_{\pm 1.7}$} & 93.9\rlap{$_{\pm 0.8}$} & 92.9\rlap{$_{\pm 1.5}$} & 91.5\rlap{$_{\pm 0.3}$} & \underline{91.2}\rlap{$_{\pm 0.5}$} & 87.4 \\
			LoRA$_{r=8}$~\cite{hu2022lora} & 1.33M & 68.9\rlap{$_{\pm0.7}$} & \textbf{90.2}\rlap{$_{\pm0.3}$} & 90.1\rlap{$_{\pm0.6}$} & 85.2\rlap{$_{\pm0.4}$} & 93.9\rlap{$_{\pm0.2}$} & 95.0\rlap{$_{\pm0.8}$} & 90.7\rlap{$_{\pm1.4}$} & 90.2\rlap{$_{\pm2.1}$} & 88.0 \\
			AdaLoRA$_{r=8}$ ~\cite{zhang2023adalora} & 1.33M & 69.4\rlap{$_{\pm0.9}$} & 89.7\rlap{$_{\pm1.0}$} & 90.2\rlap{$_{\pm1.0}$} & \underline{87.1}\rlap{$_{\pm1.5}$} & \textbf{94.1}\rlap{$_{\pm0.0}$} & 95.1\rlap{$_{\pm0.7}$} & 91.3\rlap{$_{\pm0.2}$} & 88.9\rlap{$_{\pm1.2}$} & 88.2 \\
			LoRA-Pro$_{r=8}$ \cite{wang2025lorapro} & 1.33M & \underline{70.1}\rlap{$_{\pm1.0}$} & \underline{90.0}\rlap{$_{\pm0.5}$} & 90.4\rlap{$_{\pm1.0}$} & 85.2\rlap{$_{\pm2.5}$} & \underline{94.0}\rlap{$_{\pm0.1}$} & 94.5\rlap{$_{\pm0.4}$} & 91.1\rlap{$_{\pm0.4}$} & 90.7\rlap{$_{\pm0.0}$} & 88.3 \\
			GeoLoRA \cite{schotthofer2025geolora} & \(1.75\)--\(3.10\)\,M & 68.3\rlap{$_{\pm1.0}$} & 89.8\rlap{$_{\pm0.2}$} & \underline{90.5}\rlap{$_{\pm0.6}$} & \textbf{87.2}\rlap{$_{\pm1.8}$} & 93.7\rlap{$_{\pm0.7}$} & \textbf{95.8}\rlap{$_{\pm0.2}$} & \underline{91.6}\rlap{$_{\pm0.0}$} & 90.4\rlap{$_{\pm0.2}$} & \underline{88.4} \\
			\rowcolor{dtlorapurple} DT-LoRA$_{r=8}$ & 1.33M & \textbf{70.3}\rlap{$_{\pm0.4}$} & 89.4\rlap{$_{\pm0.4}$} & \textbf{90.7}\rlap{$_{\pm0.4}$} & 86.2\rlap{$_{\pm1.1}$} & \textbf{94.1}\rlap{$_{\pm0.2}$} & \underline{95.3}\rlap{$_{\pm0.3}$} & \textbf{91.7}\rlap{$_{\pm0.2}$} & \textbf{91.8}\rlap{$_{\pm0.3}$} & \textbf{88.7} \\
			
			\hline
		\end{tabular}%
	}
\end{table*}
\subsection{Results on Natural Language Processing Tasks}
\label{sec:nlp-results}

\noindent\textbf{Commonsense Reasoning.} To assess whether the benefits extend beyond vision, we evaluate DT-LoRA on commonsense reasoning (\cref{tab:commonsense-results}).
On LLaMA3-8B, DT-LoRA reaches 85.4\% average accuracy with half the trainable parameters (0.35\% vs. 0.70\%), outperforming LoRA (80.8\%), DoRA (85.2\%), and LoRA-Pro (84.7\%). The same trend holds on LLaMA2-7B (full results in~Appendix~\cref{tab:commonsense-full}), where DT-LoRA achieves 80.0\% average accuracy, outperforming LoRA (77.6\%) and DoRA (79.7\%) with roughly half the trainable parameters.

\noindent\textbf{Natural Language Understanding.} On GLUE (\cref{tab:glue-results}), DT-LoRA achieves the highest average score of 88.7, outperforming LoRA under the same parameter budget by 0.7 points and GeoLoRA, which uses more trainable parameters, by 0.3 points. It also surpasses full fine-tuning (88.3) while training less than 1\% of the model parameters. The largest improvement over LoRA is observed on QQP, with a gain of 1.6 points. These results show that DT-LoRA extends its benefits to language tasks without increasing the trainable parameter budget.

\begin{table}[t]
	\captionsetup{skip=2pt}
	\centering
	\caption{Ablation study of DT-LoRA at a fixed total rank $r=16$.}
	\label{tab:ablation-main}
	
	\fontsize{6.5pt}{6.5pt}\selectfont
	\setlength{\tabcolsep}{3pt}
	\renewcommand{\arraystretch}{1}
	
	\setlength{\aboverulesep}{0.3pt}
	\setlength{\belowrulesep}{0.3pt}
	\setlength{\cmidrulesep}{0pt}
	
	\begin{tabularx}{0.95\linewidth}{
			l*{4}{>{\centering\arraybackslash}X}
		}
		\hline
		\multirow{2.5}{*}{Variant}
		& \multicolumn{2}{c}{OVDG-SS}
		& \multicolumn{2}{c}{MS COCO} \\
		\cmidrule(r){2-3} \cmidrule(l){4-5}
		
		& Dv-19 & Dv-58 & Det.mAP & Seg.mAP \\
		\hline
		
		w/o Controlled Write-Back
		& 27.3 & 16.8 & 33.7 & 32.9 \\
		
		w/o Dynamic Branch
		& 52.7 & 48.5 & 36.2 & 34.5 \\

		w/ Full Gradient Refresh
		& 54.2 & 51.9 & 43.0 & 39.3 \\
		
		\rowcolor{dtlorapurple} DT-LoRA (Full)
		& 59.2 & 56.1 & 44.1 & 40.3 \\
		
		\hline
	\end{tabularx}
\end{table}

\subsection{Ablation Study and Analysis}
\label{sec:ablation}

\Cref{tab:ablation-main} reports ablation results on OVDG-SS (CS-7 $\rightarrow$ Dv-19/Dv-58, using CAT-Seg with EVA02 ViT-L/14) and MS COCO. Replacing the Dynamic branch with a Stable branch of the same rank (\textit{w/o Dynamic Branch}) reduces mIoU to 52.7/48.5 on Dv-19/Dv-58 and detection mAP to 36.2, bringing performance close to that of a parameter-matched LoRA baseline. This supports the importance of residual-driven subspace refresh. Moreover, extracting candidate directions from the full gradient rather than the tangent-space residual (\textit{w/ Full Gradient Refresh}) yields lower performance than DT-LoRA, suggesting that residual directions better complement the existing tangent space. Finally, overwriting the Dynamic branch without preserving accumulated updates (\textit{w/o Controlled Write-Back}) reduces mIoU to 27.3/16.8 and detection mAP to 33.7, highlighting the importance of retaining prior updates during refresh.
\begin{figure}[t]
	\captionsetup{skip=2pt}
	\centering
	\begin{subfigure}[t]{0.49\columnwidth}
		\captionsetup{skip=2pt}
		\centering
		\includegraphics[width=\linewidth,trim=13 15 13pt 8,clip]{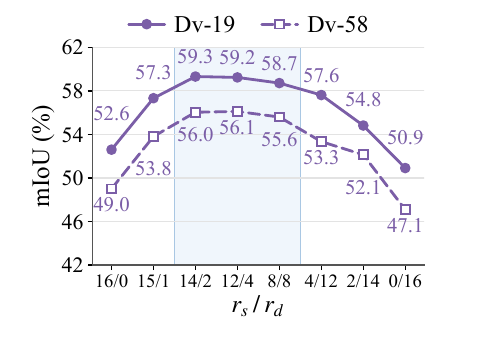}
		\caption{Rank Allocation Sensitivity}
		\label{fig:rank-sensitivity}
	\end{subfigure}\hspace{0.01\columnwidth}
	\begin{subfigure}[t]{0.49\columnwidth}
		\captionsetup{skip=2pt}
		\centering
		\includegraphics[width=\linewidth,trim=17pt 15 21 8,clip]{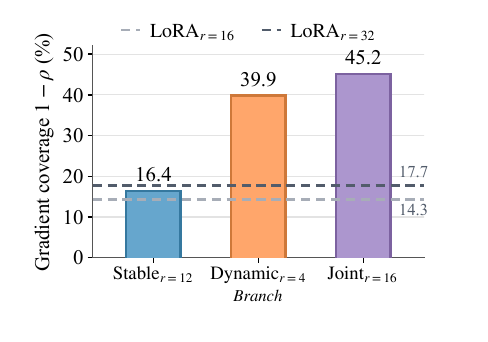}
		\caption{Branch-wise Gradient Coverage}
		\label{fig:branch-wise-gradient-coverage}
	\end{subfigure}
    \caption{Rank allocation and branch-wise gradient coverage.}
	\label{fig:ablation-analysis}
\end{figure}

We analyze rank-allocation sensitivity and branch-wise gradient coverage averaged over training on OVDG-SS (CS-7 $\rightarrow$ Dv-19/Dv-58). \cref{fig:rank-sensitivity} shows that performance remains stable across a broad range of allocations ($r_s/r_d$ from $14/2$ to $8/8$), indicating robustness to moderate rank-budget splits. Using only a single branch degrades performance, confirming that both branches are necessary. As shown in \cref{fig:branch-wise-gradient-coverage}, Dynamic ($r=4$) achieves 39.9\% coverage with merely one-third the rank of Stable, demonstrating strong rank efficiency. Their joint tangent space attains 45.2\% coverage, outperforming either branch and verifying complementary gradient subspaces. By contrast, doubling standard LoRA’s rank from 16 to 32 brings only a marginal 3.4\% gain: expanding rank alone barely boosts coverage, and embedding dominant residual directions into the update subspace is critical. Consistently, \cref{fig:gradient-coverage} shows that DT-LoRA continually extracts dominant residual directions, steadily increasing joint coverage and validating the core motivation.

\begin{table}[t]
	\captionsetup{skip=2pt}
	\centering
	\caption{Training cost on OVDG-SS (CAT-Seg with EVA02 ViT-L/14, one A100, $\tau=5$). Overhead is relative to LoRA.}
	\label{tab:dt-lora-efficiency}
	\scriptsize
	\setlength{\tabcolsep}{3pt}
	\renewcommand{\arraystretch}{1.1}
	\resizebox{\columnwidth}{!}{%
		\begin{tabular}{lcccccc}
			\hline
			Method &
			Params &
			\makecell{Extra time} &
			\makecell{Extra memory} &
			\makecell{Time (s/iter)} &
			Overhead \\
			\hline
			LoRA \cite{hu2022lora} & 4.33M & -- & -- & 0.82 & -- \\
			LoRA-Pro \cite{wang2025lorapro} & 4.33M & $\mathcal{O}(dr^2)$ & $\mathcal{O}(d_{\min}d)$ & 1.16 & $+41.5\%$ \\
			Naive (Full $G_{\mathrm{res}}$ + SVD)& 4.33M & $\mathcal{O}((N{+}d_{\min})d_{\min}d/\tau)$ & $\mathcal{O}(d_{\min}d)$ & 4.55 & $+454.9\%$ \\
			
			\cellcolor{dtlorapurple}DT-LoRA & \cellcolor{dtlorapurple}4.33M & \cellcolor{dtlorapurple}$\mathcal{O}(Nqd/\tau)$ & \cellcolor{dtlorapurple}$\mathcal{O}((N{+}d)q)$ & \cellcolor{dtlorapurple}0.89 & \cellcolor{dtlorapurple}$+8.5\%$ \\
			\hline
		\end{tabular}%
	}
\end{table}

\subsection{Efficiency and Deployment}
\label{sec:efficiency}

We measure the training cost on OVDG-SS (CS-7) in \cref{tab:dt-lora-efficiency}. DT-LoRA adds only about 70~ms per iteration ($0.89$ vs.\ $0.82$~s, $+8.5\%$). Under the same parameter budget, this overhead is also substantially lower than LoRA-Pro ($+41.5\%$). In contrast, explicitly forming $G_{\mathrm{res}}$ and computing its full SVD (Naive) increases the training time by $454.9\%$. This gap comes from combining the two components: randomized SVD accesses $G_{\mathrm{res}}$ only through matrix products, and the matrix-free operators evaluate these products directly from the cached $X$ and $R$. Hence $G_{\mathrm{res}}$ is never formed, and the extra time and memory reduce to $\mathcal{O}(Nqd/\tau)$ and $\mathcal{O}((N+d)q)$, with $N$ cached tokens and $q\ll d_{\min}$. We further evaluate training memory on LLaMA3-8B under the same experimental settings. DT-LoRA achieves the best performance with only a negligible memory overhead compared with LoRA. Complete results are provided in Appendix~\cref{fig:memory-accuracy}. After training, both branches are merged into $W_0$, so the deployed model has the same architecture and inference cost as merged LoRA.

\section{Conclusion}
\label{sec:conclusion}

In this paper, we presented DT-LoRA, a parameter-efficient adaptation method that addresses the tangent-space constraint on factorized LoRA's first-order updates. Our analysis showed that the gradient components outside the current tangent space carry substantial energy concentrated in a few dominant directions, motivating their periodic incorporation within a fixed rank budget. DT-LoRA combines matrix-free residual direction extraction with controlled write-back to adapt the trainable subspace while limiting refresh-induced perturbations. Extensive experiments across vision and language tasks demonstrated consistent improvements over LoRA and its variants without additional trainable parameters or inference cost. These findings highlight the value of adapting low-rank update subspaces to previously unexploited gradient information.

\section{LLM Usage Statement}
We used GPT-5.6 and DeepSeek V4.1 Flash for grammar checking and language
polishing, mainly to improve the clarity and readability of the Introduction
and Related Work sections.
To prepare the appendix and facilitate the future release of
the code as open-source software, we used these tools to restructure the code
and add comments. This was intended to make the implementation easier to read
and understand.
The LLMs were not involved in developing the core research insights,
formulating the research ideas, designing the method, deriving the theory, or
planning the detailed experiments. All research content, experiments, results,
and conclusions were produced, checked, and confirmed by the authors.

{\small
\bibliographystyle{ieeenat_fullname}
\bibliography{main}
}

\clearpage
\twocolumn[\centering\Large\textbf{Appendix}\par\vspace{1em}]
\appendix
\renewcommand{\thesubsection}{\thesection\arabic{subsection}}
\setcounter{figure}{0}
\setcounter{table}{0}
\setcounter{equation}{0}
\renewcommand{\thefigure}{S\arabic{figure}}
\renewcommand{\thetable}{S\arabic{table}}
\renewcommand{\theequation}{S\arabic{equation}}
\setcounter{algorithm}{0}
\setcounter{footnote}{0}
\counterwithout{proposition}{section}
\setcounter{proposition}{0}
\renewcommand{\theproposition}{S\arabic{proposition}}
\renewcommand*{\theHfigure}{supp.\arabic{figure}}
\renewcommand*{\theHtable}{supp.\arabic{table}}
\renewcommand*{\theHequation}{supp.\arabic{equation}}
\providecommand*{\theHalgorithm}{\arabic{algorithm}}
\renewcommand*{\theHalgorithm}{supp.\arabic{algorithm}}
\captionsetup[figure]{skip=2pt}
\captionsetup[table]{skip=2pt}

\section{Training Procedure}
\label{app:algorithm}

\Cref{alg:dt-lora} summarizes the training procedure of DT-LoRA. Both branches are updated by the standard optimizer, and every $\tau$ steps the Dynamic branch is refreshed: the principal residual directions are extracted by Matrix-Free Randomized SVD and merged into the branch through Controlled Write-Back.

\begin{algorithm}[H]
	\caption{DT-LoRA training with periodic dynamic refresh.}
	\label{alg:dt-lora}
	\small
	\begin{algorithmic}[1]
		\Require Stable factors $(A_s,B_s)$, Dynamic factors $(A_d,B_d)$, refresh interval $\tau$, Dynamic rank $q$
		\Ensure Updated Stable and Dynamic factors
		\For{each training step $t$}
		\State Compute the loss with $W=W_0+s(B_sA_s+B_dA_d)$ and back-propagate
		\State Update all factors with the ordinary optimizer
		\If{$t \bmod \tau = 0$}
		\State Obtain matched activations $X$ and output gradients $R$
		\State Extract $H\approx-[G_{\mathrm{res}}]_q$ by Matrix-Free Randomized SVD without forming $G_{\mathrm{res}}$
		\State $D_d^{\mathrm{old}}\gets sB_dA_d$, $D_d^{\mathrm{new}}\gets\Pi_q(D_d^{\mathrm{old}}+\beta H)$ with the bounded refresh scale $\beta$
		\State Halve $\beta$ and recompute $D_d^{\mathrm{new}}$ until the jump constraint holds or the backtracking limit is reached
		\If{the jump constraint holds}
		\State Factorize $D_d^{\mathrm{new}}=sB_dA_d$ and reset only Dynamic optimizer states
		\EndIf
		\State Update $D_d^{\mathrm{anchor}}\gets sB_dA_d$ after every scheduled refresh
		\EndIf
		\EndFor
		\State \Return $(A_s,B_s,A_d,B_d)$
	\end{algorithmic}
\end{algorithm}

\section{Proof of Proposition~3.1}
\label{app:proof-optimal}

\begin{proof}
	Completing the square gives
	\begin{equation}
		\left\langle G_{\mathrm{res}},\Delta\right\rangle_F
		+\frac{1}{2\eta}\|\Delta\|_F^2
		=\frac{1}{2\eta}\|\Delta+\eta G_{\mathrm{res}}\|_F^2
		-\frac{\eta}{2}\|G_{\mathrm{res}}\|_F^2.
		\label{eq:complete-square}
	\end{equation}
	Dropping the constant term, \cref*{eq:residual-quadratic} of the main text is equivalent to
	\begin{equation}
		\min_{\substack{\operatorname{rank}(\Delta)\leq q\\
				\Delta\in\mathcal{T}^{\perp}}}
		\|\Delta+\eta G_{\mathrm{res}}\|_F^2.
		\label{eq:equiv-problem}
	\end{equation}
	By the Eckart--Young theorem~\cite{eckart1936approximation}, the optimal solution without the constraint $\Delta\in\mathcal{T}^{\perp}$ is $-\eta[G_{\mathrm{res}}]_q$. Since $G_{\mathrm{res}}\in\mathcal{T}^{\perp}$, its left and right singular vectors lie in the ranges of $P_B^{\perp}$ and $P_A^{\perp}$, respectively. 
	Hence $[G_{\mathrm{res}}]_q\in\mathcal{T}^{\perp}$, making $-\eta[G_{\mathrm{res}}]_q$ feasible and therefore the optimal solution of \cref*{eq:residual-quadratic} of the main text.
\end{proof}

\section{Properties of Controlled Write-Back}
\label{app:writeback}

This section establishes the properties of the two operations in the controlled write-back mechanism: merge-and-compress keeps the refresh perturbation bounded and, under a spectral-gap condition, preserves the accumulated principal subspace, while the jump constraint maintains training stability. Following the notation of \cref*{sec:dynamic-refresh} of the main text, let $E=D_d^{\mathrm{new}}-D_d^{\mathrm{old}}$. Since the refresh changes only the Dynamic branch, $E$ is exactly the weight change $W^{\mathrm{new}}-W^{\mathrm{old}}$ caused by the refresh.

\begin{proposition}[Perturbation bound of merge-and-compress]
	\label{prop:bounded-refresh}
	Suppose $\operatorname{rank}(D_d^{\mathrm{old}})\leq q$, $D_d^{\mathrm{new}}$ is given by \cref*{eq:dynamic-compress} of the main text, and $\beta$ is given by \cref*{eq:refresh-scale} of the main text. Then
	\begin{equation}
		\|E\|_F\leq 2\beta\|H\|_F\leq 2\rho_{\mathrm{refresh}}w_d,
		\label{eq:refresh-bound}
	\end{equation}
	and $|\sigma_i(D_d^{\mathrm{new}})-\sigma_i(D_d^{\mathrm{old}})|\leq 2\beta\|H\|_F$ for any $i\leq q$.
\end{proposition}

\begin{proof}
	Let $M=D_d^{\mathrm{old}}+\beta H$. Since $D_d^{\mathrm{old}}$ is an approximation of $M$ with rank at most $q$ and $\Pi_q(M)$ is the best such approximation,
	\begin{equation}
		\|M-D_d^{\mathrm{new}}\|_F\leq\|M-D_d^{\mathrm{old}}\|_F=\beta\|H\|_F.
		\label{eq:proof-bound-1}
	\end{equation}
	By the triangle inequality,
	\begin{equation}
		\|E\|_F\leq\|D_d^{\mathrm{new}}-M\|_F+\|M-D_d^{\mathrm{old}}\|_F\leq 2\beta\|H\|_F.
		\label{eq:proof-bound-2}
	\end{equation}
	By \cref*{eq:refresh-scale} of the main text, the refresh scale satisfies $\beta\|H\|_F=\rho_{\mathrm{refresh}}w_d\frac{\|H\|_F}{\|H\|_F+\varepsilon}\leq\rho_{\mathrm{refresh}}w_d$, where the inequality follows from $\varepsilon>0$. The singular-value bound follows from Weyl's inequality,
	\begin{equation}
		|\sigma_i(D_d^{\mathrm{old}}+E)-\sigma_i(D_d^{\mathrm{old}})|\leq\|E\|_2\leq\|E\|_F.
		\label{eq:weyl}
	\end{equation}
\end{proof}
\begin{figure*}[t]
	\captionsetup{skip=2pt}
	\centering
	\begin{subfigure}[t]{0.32\textwidth}
		\centering
		\includegraphics[width=\linewidth]{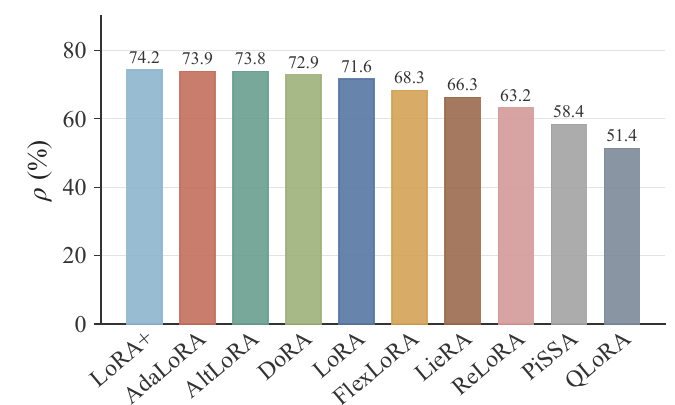}
		\caption{Residual-gradient energy ratio $\rho$.}
		\label{fig:coco-residual-rho}
	\end{subfigure} 
	\hfill
	\begin{subfigure}[t]{0.32\textwidth}
		\centering
		\includegraphics[width=\linewidth]{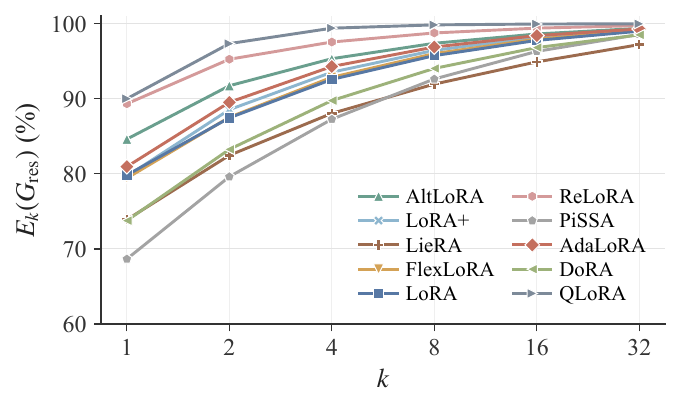}
		\caption{Top-$k$ residual spectral energy $E_k$.}
		\label{fig:coco-residual-energy}
	\end{subfigure}
	\hfill
	\begin{subfigure}[t]{0.32\textwidth}
		\centering
		\includegraphics[width=\linewidth]{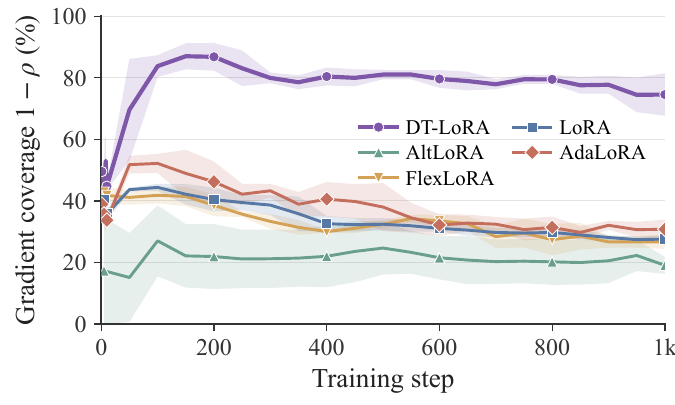}
		\caption{Gradient coverage during training.}
		\label{fig:coco-gradient-coverage}
	\end{subfigure}
	\caption{Empirical analysis of residual gradients over training iterations for object detection and instance segmentation on MS COCO.}
	\label{fig:coco-residual-gradient-analysis}
\end{figure*}
\cref{eq:refresh-bound} shows that the refresh perturbation depends only on the injection magnitude: it is of the same order as the cumulative update of the branch within one refresh interval and is independent of $\|D_d^{\mathrm{old}}\|_F$. When $2\beta\|H\|_F$ is smaller than the leading singular values of $D_d^{\mathrm{old}}$, these singular values undergo only a bounded relative change. Moreover, by Wedin's $\sin\Theta$ theorem, the deviation between the top-$k$ singular subspaces of $D_d^{\mathrm{old}}$ and $D_d^{\mathrm{new}}$ is $\mathcal{O}\bigl(\beta\|H\|_F/(g_k-2\beta\|H\|_F)\bigr)$ for any $k\leq q$ with spectral gap $g_k=\sigma_k(D_d^{\mathrm{old}})-\sigma_{k+1}(D_d^{\mathrm{old}})>2\beta\|H\|_F$. Hence the accumulated principal directions are preserved whenever the injection is small relative to the spectral gap. In contrast, direct replacement $D_d^{\mathrm{new}}=\beta H$ gives $\|E\|_F\geq\|D_d^{\mathrm{old}}\|_F-\beta\|H\|_F$: the perturbation grows with the accumulated state, and all existing directions are discarded.

\begin{proposition}[Jump constraint and training stability]
	\label{prop:jump-stability}
	Suppose the loss $\mathcal{L}$ is $L$-smooth with respect to $W$, \ie, $\nabla\mathcal{L}$ is $L$-Lipschitz. If the refresh satisfies \cref*{eq:jump-gate} of the main text, then with $\delta=\rho_{\mathrm{jump}}(\|D_d^{\mathrm{old}}\|_F+\varepsilon)$,
	\begin{equation}
		\|\nabla\mathcal{L}(W^{\mathrm{new}})-\nabla\mathcal{L}(W^{\mathrm{old}})\|_F\leq L\delta,
		\label{eq:grad-stability}
	\end{equation}
	\begin{equation}
		\bigl|\mathcal{L}(W^{\mathrm{new}})-\mathcal{L}(W^{\mathrm{old}})\bigr|
		\leq\|\nabla\mathcal{L}(W^{\mathrm{old}})\|_F\,\delta+\frac{L}{2}\delta^{2}.
		\label{eq:loss-stability}
	\end{equation}
	Moreover, the backtracking procedure that reduces $\beta$ by a fixed ratio terminates in a finite number of steps.
\end{proposition}

\begin{proof}
	By \cref*{eq:jump-gate} of the main text, $\|W^{\mathrm{new}}-W^{\mathrm{old}}\|_F=\|E\|_F\leq\delta$. \cref{eq:grad-stability} then follows directly from the Lipschitz continuity of $\nabla\mathcal{L}$. By the quadratic bound implied by $L$-smoothness,
	\begin{equation}
		\bigl|\mathcal{L}(W^{\mathrm{new}})-\mathcal{L}(W^{\mathrm{old}})-\langle\nabla\mathcal{L}(W^{\mathrm{old}}),E\rangle_F\bigr|\leq\frac{L}{2}\|E\|_F^{2},
	\end{equation}
	and \cref{eq:loss-stability} follows from the Cauchy--Schwarz inequality. For the backtracking procedure, \cref{prop:bounded-refresh} gives $\|E\|_F\leq 2\beta\|H\|_F\to 0$ as $\beta\to 0$, whereas the right-hand side satisfies $\delta\geq\rho_{\mathrm{jump}}\varepsilon>0$. Hence the constraint must hold after finitely many reductions of $\beta$.
\end{proof}

The smoothness assumption is only needed along the segment between $W^{\mathrm{old}}$ and $W^{\mathrm{new}}$. The two operations of controlled write-back act at different scales. By \cref{prop:bounded-refresh}, merge-and-compress alone bounds $\|E\|_F$ by $2\rho_{\mathrm{refresh}}w_d$, which scales with the recent cumulative update, whereas the jump constraint bounds $\|E\|_F$ relative to the current branch magnitude $\|D_d^{\mathrm{old}}\|_F$. Hence the jump constraint holds without backtracking whenever $2\rho_{\mathrm{refresh}}w_d\leq\rho_{\mathrm{jump}}(\|D_d^{\mathrm{old}}\|_F+\varepsilon)$. It becomes active only when the recent update is large relative to the accumulated branch, \eg, early in training when $D_d^{\mathrm{old}}$ is still small. A detailed ablation study of the jump constraint is provided in \cref{app:jump-gate-ablation}.

\section{Additional Experiments and Analyses}
\label{app:additional-results}

\subsection{Empirical Analysis on Additional Tasks}
\label{app:coco-residual-analysis}

To examine whether the residual-gradient observations in the main text extend beyond semantic segmentation, we conduct the same empirical analysis on object detection and instance segmentation using MS COCO. Specifically, we measure the residual-gradient energy ratio $\rho$, the cumulative top-$k$ spectral energy $E_k$, and the evolution of tangent-space gradient coverage $1-\rho$ during training. The results are summarized in \cref{fig:coco-residual-gradient-analysis}.

The results show trends consistent with those observed on OVDG-SS. As shown in \cref{fig:coco-residual-rho}, a substantial fraction of the gradient energy still lies outside the tangent spaces of existing LoRA variants. This indicates that their first-order factor updates leave a large amount of gradient information unexploited. Meanwhile, \cref{fig:coco-residual-energy} shows that the residual energy is concentrated in only a few dominant singular directions, with the first two singular components accounting for more than 80\% of the residual energy. This observation supports the use of a low-rank approximation to recover the gradient components that are inaccessible to the current tangent space. Finally, \cref{fig:coco-gradient-coverage} shows that DT-LoRA steadily improves gradient coverage during training. Its coverage rapidly rises to around 90\% early on and stabilizes near 80\% as training progresses, whereas other LoRA variants maintain substantially lower coverage throughout training. These results confirm that the two key properties motivating DT-LoRA, substantial residual energy and strong low-rank concentration, also hold across additional tasks. Moreover, these properties become particularly evident when the downstream transfer objective differs substantially from the pre-training objective.

\begin{figure*}[t]
	\centering
	\captionsetup{font=small,skip=2pt}
	\includegraphics[width=0.9\linewidth]{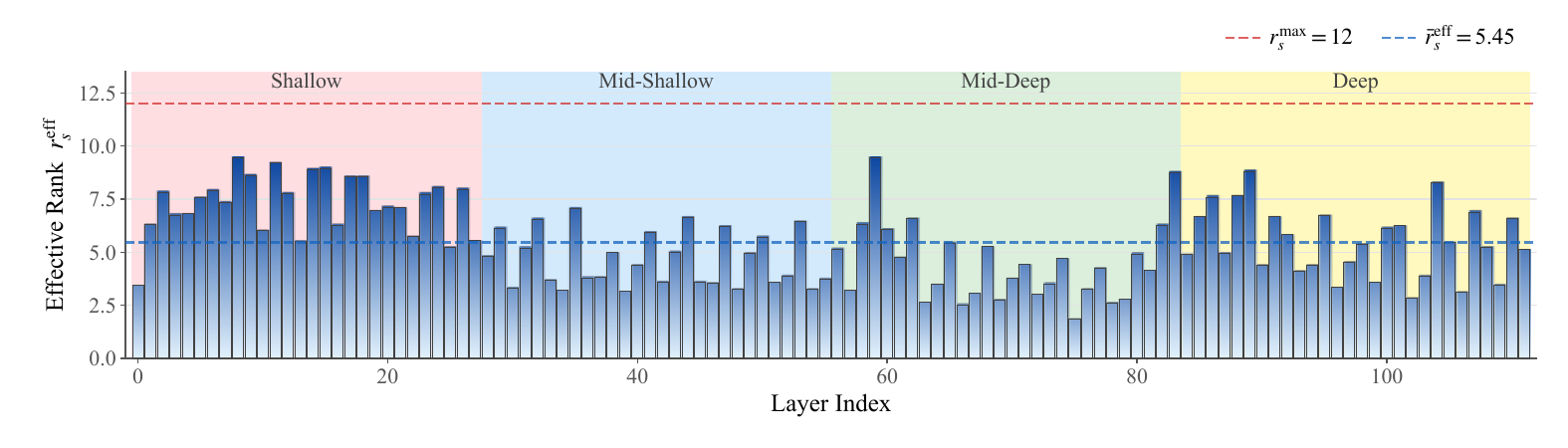}
	\caption{Effective rank distribution across 112 layers for the Stable branch ($r_s=12$). Layers are grouped by depth: Shallow (0-27), Mid-Shallow (28-55), Mid-Deep (56-83), and Deep (84-111). The Stable branch shows layer-dependent utilization with 45.4\% average capacity usage, exhibiting substantial variation across different depths.}
	\label{fig:stable-effective-rank}
\end{figure*}

\begin{figure*}[t]
	\centering
	\captionsetup{font=small,skip=2pt}
	\includegraphics[width=0.9\linewidth]{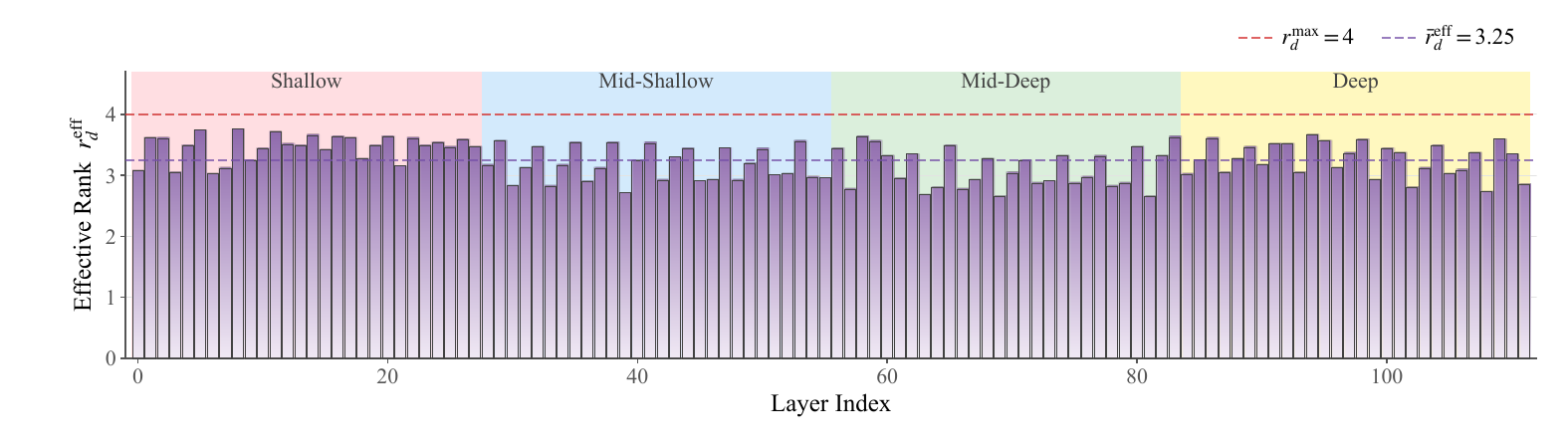}
	\caption{Effective rank distribution across 112 layers for the Dynamic branch ($r_d=4$). The branch maintains consistent near-saturation utilization (81.2\% average) across all layers with a standard deviation of only 0.30, indicating uniform demand for dynamic adaptation capacity throughout the network.}
	\label{fig:dynamic-effective-rank}
\end{figure*}

\begin{figure}[t]
    \centering
    \includegraphics[width=0.85\linewidth]{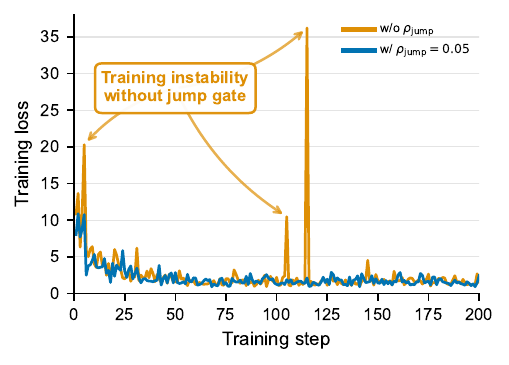}
    \caption{Jump gate ablation. Training loss over 200 optimizer steps. Without $\rho_{\mathrm{jump}}$, training exhibits severe instability with repeated loss spikes, while with $\rho_{\mathrm{jump}}=0.05$, training maintains smooth convergence.}
    \label{fig:writeback-stability-ablation}
\end{figure}

\subsection{Effective Rank Analysis Across Layers}
\label{app:effective-rank}

To understand how the two branches utilize their allocated rank capacity, we analyze the effective rank distribution across all 112 adapted layers in the ConvNeXtV2-B backbone with the configuration $r_s/r_d=12/4$ (matching the main experiments). The effective rank $r^{\mathrm{eff}} = \exp(H(\sigma))$ is defined from the Shannon entropy $H(\sigma) = -\sum_i p_i \log p_i$ of normalized squared singular values $p_i = \sigma_i^2 / \sum_j \sigma_j^2$. This metric ranges from 1 to $r_{\mathrm{max}}$: when all singular values are equal, $r^{\mathrm{eff}} = r_{\mathrm{max}}$ (full utilization); when few dominate, $r^{\mathrm{eff}} \ll r_{\mathrm{max}}$ (rank redundancy).

\cref{fig:stable-effective-rank,fig:dynamic-effective-rank} show the results averaged over sampled training steps on MS-COCO. The Stable branch ($r_s=12$) exhibits cross-layer heterogeneity: shallow layers (0-27) average 7.26, mid-layers (28-83) drop to 4.47--4.62, and deep layers (84-111) recover to 5.46, with an overall average of 5.45 (45.4\% utilization). In contrast, the Dynamic branch ($r_d=4$) maintains consistent near-saturation: network-wide average 3.25 (81.2\%) with standard deviation 0.30, all layers within 2.65--3.77. These patterns reveal functional specialization. The Stable branch handles long-term representations with layer-dependent needs, leading to heterogeneity and partial redundancy. The Dynamic branch captures rapid changes with uniform demand and near-full utilization. This suggests that future work could explore the potential benefits of layer-specific allocation and flexible rank adaptation for both the Stable and Dynamic branches.

\subsection{Jump Gate Ablation Analysis}
\label{app:jump-gate-ablation}

To verify the stabilizing effect of the jump gate constraint in controlled write-back, we conduct a more detailed mechanism-level ablation. Specifically, we compare two configurations on MS COCO: standard DT-LoRA with $\rho_{\mathrm{jump}}=0.05$ and \textit{w/o $\rho_{\mathrm{jump}}$} (no constraint), while keeping all other settings consistent with the main experiments ($r_s=12$, $r_d=4$, $\alpha=16$, refresh interval $\tau=5$).
\cref{fig:writeback-stability-ablation} shows the training loss curves. For \textit{w/o $\rho_{\mathrm{jump}}$}, the loss exhibits severe instability with large spikes, peaking at 36.2. These spikes occur at refresh events where the Dynamic branch update exceeds the relative-change bound, causing abrupt perturbations to the model weights. With $\rho_{\mathrm{jump}}=0.05$, training converges smoothly throughout with no visible spikes. This result directly confirms that the jump gate constraint in \cref{eq:jump-gate} is necessary for maintaining training continuity. Although the candidate directions $H$ are optimal in the residual space (\Cref{prop:optimal-residual-update}), injecting them without control destabilizes training.

\subsection{Training Efficiency}
\label{app:memory-efficiency}

To better understand where this efficiency comes from, we examine the two complementary components used to compute the residual-gradient subspace. \cref{tab:efficiency-full} reports the complete results. Matrix-Free computation directly computes the required matrix-vector products instead of explicitly materializing the full residual gradient. When used alone, it slightly reduces the per-iteration time from 4.55 to 4.30~s/iter. Randomized SVD alone extracts the principal residual directions without performing a full SVD, reducing the per-iteration time to 0.93~s/iter. More importantly, Matrix-Free computation supports randomized SVD by evaluating its required matrix products directly from cached activations and output gradients, without materializing the residual gradient. With both components enabled, the full DT-LoRA implementation further reduces the time to 0.89~s/iter, corresponding to only 8.5\% overhead over LoRA, compared with 454.9\% for directly materializing the full residual gradient and applying SVD. Together, the two components reduce the extra memory complexity to $\mathcal{O}((N{+}d)q)$ while substantially reducing the cost of extracting the residual directions.

\begin{table}[H]
	\centering
	\caption{Full efficiency breakdown of DT-LoRA on OVDG-SS (CAT-Seg with EVA02 ViT-L/14, one A100, $\tau=5$). All variants have 4.33M trainable parameters. Overhead is relative to LoRA.}
	\label{tab:efficiency-full}
	\scriptsize
	\setlength{\tabcolsep}{3pt}
	\renewcommand{\arraystretch}{1.2}
	\resizebox{\columnwidth}{!}{%
		\begin{tabular}{lcccccc}
			\hline
			Method &
			\makecell{Matrix-\\free} &
			\makecell{Randomized\\SVD} &
			\makecell{Extra\\time} &
			\makecell{Extra\\memory} &
			\makecell{Time\\(s/iter)} &
			Overhead \\
			\hline
			LoRA \cite{hu2022lora} & & & -- & -- & 0.82 & -- \\
			\hline
			\multirow{4}{*}{DT-LoRA}
			& & & $\mathcal{O}((N{+}d_{\min})d_{\min}d/\tau)$ & $\mathcal{O}(d_{\min}d)$ & 4.55 & $+454.9\%$ \\
			& $\checkmark$ & & $\mathcal{O}((N{+}d_{\min})d_{\min}d/\tau)$ & $\mathcal{O}((N{+}d)d_{\min})$ & 4.30 & $+424.4\%$ \\
			& & $\checkmark$ & $\mathcal{O}(Nd_{\min}d/\tau)$ & $\mathcal{O}(d_{\min}d)$ & 0.93 & $+13.4\%$ \\
			& \cellcolor{dtlorapurple}$\checkmark$ & \cellcolor{dtlorapurple}$\checkmark$ & \cellcolor{dtlorapurple}$\mathcal{O}(Nqd/\tau)$ & \cellcolor{dtlorapurple}$\mathcal{O}((N{+}d)q)$ & \cellcolor{dtlorapurple}0.89 & \cellcolor{dtlorapurple}$+8.5\%$ \\
			\hline
		\end{tabular}%
	}
\end{table}

\begin{table*}[t]
	\caption{Comparison of different methods with EVA02 ViT-B/16 under the Real-to-Real OVDG-SS setting.}
	\label{tab:cs7-full}
	\centering
	\setlength{\tabcolsep}{2.5pt}
	\resizebox{\textwidth}{!}{
		\begin{tabular}{
				@{}
				l<{\hspace{4pt}} @{}
				>{\hspace{4pt}}c<{\hspace{4pt}} @{}
				>{\hspace{4pt}}c<{\hspace{4pt}} @{}
				>{\hspace{4pt}}c<{\hspace{8pt}} @{} |
				*{4}{c} |
				*{5}{c}
				@{}
			}
			\hline
			\multirow{2}{*}{Method} & \multirow{2}{*}{Backbone} & \multirow{2}{*}{Params} & \multirow{2}{*}{\makecell{Training\\Data}} & \multicolumn{4}{c|}{Dv-19} & \multicolumn{5}{c}{Dv-58} \\
			& & & & \multicolumn{1}{c}{ACDC-19} & \multicolumn{1}{c}{BDD-19} & \multicolumn{1}{c}{Mapi-19} & \multicolumn{1}{c|}{Avg.} & \multicolumn{1}{c}{ACDC-41} & \multicolumn{1}{c}{BDD-41} & \multicolumn{1}{c}{Mapi-30} & \multicolumn{1}{c}{RW-10} & \multicolumn{1}{c}{Avg.} \\
			\hline
			CAT-Seg (CVPR'24) \cite{cho2024cat} & \multirow{7}{*}{\makecell{EVA02 \\ ViT-B/16}} & 23.64M & CS-7 & 38.9 & 44.0 & 47.6 & 43.5 & 47.6 & 51.6 & 38.2 & 36.5 & 43.5 \\
			SoMA+CAT-Seg \cite{yun2025soma} & & 1.97M & CS-7 & 44.5 & 46.6 & \underline{53.8} & 48.3 & 52.3 & 57.5 & 42.0 & \textbf{39.7} & 47.9 \\
			DoRA\smash{$_{\scalebox{0.7}{\text{CAT-Seg}}}$} \cite{liu2024dora} & & 2.03M & CS-7 & \underline{44.9} & 46.5 & \underline{53.8} & 48.4 & \underline{53.3} & 57.8 & \underline{42.6} & 37.4 & 47.8 \\
			LoRA\smash{$_{\scalebox{0.7}{\text{CAT-Seg}}}$} \cite{hu2022lora} & & 1.97M & CS-7 & 44.2 & 46.5 & 53.2 & 48.0 & 52.7 & 58.2 & 42.4 & \textbf{39.7} & \underline{48.3} \\
			LoRA-Pro\smash{$_{\scalebox{0.7}{\text{CAT-Seg}}}$} \cite{wang2025lorapro} & & 1.97M & CS-7 & 43.5 & 47.8 & 53.2 & 48.2 &50.0 & 56.4 & 41.9&38.5 & 46.7  \\
			S\smash{$^{\scalebox{0.7}{2}}$}-Corr (CVPR'26) \cite{zhao2026s2corr} & & 23.64M & CS-7 & 44.6 & \textbf{50.1} & \textbf{56.2} & \underline{50.3} & 52.3 & \underline{58.6} & 42.0 & 38.6 & 47.9 \\
			\rowcolor{dtlorapurple}  DT-LoRA\smash{$_{\scalebox{0.7}{\text{CAT-Seg}}}$} & & 1.97M & CS-7 & \dtGain{\textbf{48.5}}{9.6} & \dtGain{\underline{48.6}}{4.6} & \dtGain{\textbf{56.2}}{8.6} & \dtGain{\textbf{51.1}}{7.6} & \dtGain{\textbf{55.3}}{7.7} & \dtGain{\textbf{59.5}}{7.9} & \dtGain{\textbf{44.1}}{5.9} & \dtGain{\underline{39.0}}{2.5} & \dtGain{\textbf{49.5}}{6.0} \\

			\hline
		\end{tabular}
	}
\end{table*}

\begin{table*}[t]
	\caption{Comparison of different methods with EVA02 ViT-B/16 under the Synthetic-to-Real OVDG-SS setting.}
	\label{tab:gta7-full}
	\centering
	\setlength{\tabcolsep}{2.5pt}
	\resizebox{\textwidth}{!}{
		\begin{tabular}{
				@{}
				l<{\hspace{4pt}} @{}
				>{\hspace{4pt}}c<{\hspace{4pt}} @{}
				>{\hspace{4pt}}c<{\hspace{4pt}} @{}
				>{\hspace{4pt}}c<{\hspace{8pt}} @{} |
				*{5}{c} |
				*{5}{c}
				@{}
			}
			\hline
			\multirow{2}{*}{Method} & \multirow{2}{*}{Backbone} & \multirow{2}{*}{Params} & \multirow{2}{*}{\makecell{Training\\Data}} & \multicolumn{5}{c|}{Dv-19} & \multicolumn{5}{c}{Dv-58} \\
			& & & & \multicolumn{1}{c}{CS-19} & \multicolumn{1}{c}{ACDC-19} & \multicolumn{1}{c}{BDD-19} & \multicolumn{1}{c}{Mapi-19} & \multicolumn{1}{c|}{Avg.} & \multicolumn{1}{c}{ACDC-41} & \multicolumn{1}{c}{BDD-41} & \multicolumn{1}{c}{Mapi-30} & \multicolumn{1}{c}{RW-10} & \multicolumn{1}{c}{Avg.} \\
			\hline
			CAT-Seg (CVPR'24) \cite{cho2024cat} &\multirow{7}{*}{\makecell{EVA02 \\ ViT-B/16}} & 23.64M & GTA-7 & 43.7 & 37.6 & 45.4 & 48.8 & 43.9 & 51.2 & 56.2 & 39.3 & \underline{35.7} & 45.6 \\
			
			SoMA+CAT-Seg \cite{yun2025soma} & & 1.97M & GTA-7 & 44.7 & 39.8 & 45.6 & 52.1 & 45.6 & 50.6 & 56.2 & 41.4 & 30.5 & 44.7 \\
			LoRA \smash{$_{\scalebox{0.7}{\text{CAT-Seg}}}$} \cite{hu2022lora} & & 1.97M & GTA-7 & 44.6 & 40.1 & 45.8 & 51.7 & 45.6 & 51.7 & 57.0 & 41.1 & 31.9 & 45.4 \\
			LoRA-Pro\smash{$_{\scalebox{0.7}{\text{CAT-Seg}}}$} \cite{wang2025lorapro} & & 1.97M & GTA-7 & 42.4 & 38.9 & 46.9 & 50.7 & 44.7 & 52.2 & \underline{57.9}& \textbf{42.5} & 32.1& \underline{46.2} \\
			
			DoRA\smash{$_{\scalebox{0.7}{\text{CAT-Seg}}}$} \cite{liu2024dora} & & 2.03M & GTA-7 & 45.3 & 40.1 & 46.7 & 52.4 & 46.1 & \underline{52.4} & 57.7 & 41.5 & 32.8 & 46.1 \\
			S\smash{$^{\scalebox{0.7}{2}}$}-Corr (CVPR'26) \cite{zhao2026s2corr} & & 23.64M & GTA-7 & \textbf{49.3} & \underline{40.4} & \textbf{48.6} & \textbf{54.5} & \textbf{48.2} & 51.3 & 56.4 & \underline{42.4} & \textbf{36.7} & \textbf{46.7} \\
			\rowcolor{dtlorapurple}  DT-LoRA\smash{$_{\scalebox{0.7}{\text{CAT-Seg}}}$} & & 1.97M & GTA-7 & \dtGain{\underline{47.4}}{3.7} & \dtGain{\textbf{41.3}}{3.7} & \dtGain{\underline{47.8}}{2.4} & \dtGain{\underline{53.8}}{5.0} & \dtGain{\underline{47.6}}{3.7} & \dtGain{\textbf{53.1}}{1.9} & \dtGain{\textbf{58.3}}{2.1} & \dtGain{\underline{42.4}}{3.1} & \dtLoss{32.9}{-2.8} & \dtGain{\textbf{46.7}}{1.1} \\
			\hline
		\end{tabular}
	}
\end{table*}
\begin{table*}[t]
	\captionsetup{skip=2pt}
	\centering
	\caption{Full accuracy (\%) on eight Commonsense170K evaluation benchmarks with LLaMA2-7B.}
	\label{tab:commonsense-full}
	\fontsize{7pt}{7pt}\selectfont
	\renewcommand{\arraystretch}{1.0}
	\setlength{\tabcolsep}{3pt}
	\begin{tabularx}{\textwidth}{llc*{9}{>{\centering\arraybackslash}X}}
			\hline
			Model & Method & Params & BoolQ & PIQA & SIQA & HellaS. & WinoG. & ARC-e & ARC-c & OBQA & Avg. \\
			\hline
			ChatGPT & - & -- & 73.1 & 85.4 & 68.5 & 78.5 & 66.1 & 89.8 & 79.9 & 74.8 & 77.0 \\
			\hline
			\multirow{6.5}{*}{LLaMA2-7B}
			& PiSSA$_{r=32}$ \cite{meng2024pissa} & 0.83\% & 67.6 & 78.1 & 78.4 & 76.6 & 78.0 & 75.8 & 60.2 & 75.6 & 73.8 \\
			& LoRA$_{r=32}$ \cite{hu2022lora} & 0.83\% & 69.8 & 79.9 & \textbf{79.5} & 83.6 & 82.6 & 79.8 & 64.7 & 81.0 & 77.6 \\
			& AdaLoRA$_{r=32}$ \cite{zhang2023adalora} & 0.83\% & 64.5 & \underline{83.6} & 76.7 & 88.9 & \textbf{85.1} & \underline{83.1} & \underline{68.5} & 79.2 & 78.7 \\
			& LoRA-Pro$_{r=32}$ \cite{wang2025lorapro} & 0.83\% & 70.9 & 83.0 & 78.8 & 88.4 & \underline{82.8} & 83.0 & 68.4 & 79.1 & 79.3 \\
			& DoRA$_{r=32}$ \cite{liu2024dora} & 0.84\% & \underline{71.8} & \textbf{83.7} & 76.0 & \underline{89.1} & 82.6 & \textbf{83.7} & 68.2 & \textbf{82.4} & \underline{79.7} \\
			\rowcolor{dtlorapurple} & DT-LoRA$_{r=16}$ & 0.43\% & \textbf{72.2} & 82.9 & \underline{79.0} & \textbf{90.0} & \underline{82.8} & 82.5 & \textbf{69.2} & \underline{81.2} & \textbf{80.0} \\
			\hline
	\end{tabularx}
\end{table*}

\cref{fig:memory-accuracy} compares the average commonsense reasoning accuracy and steady-state training memory of LoRA, LoRA-Pro, AdaLoRA, DoRA, and DT-LoRA on LLaMA3-8B. All methods are evaluated under the same training settings. DT-LoRA achieves the highest average accuracy of $85.4\%$, outperforming LoRA, AdaLoRA, LoRA-Pro, and DoRA, whose accuracies are $80.8\%$, $84.5\%$, $84.7\%$, and $85.2\%$, respectively. At the same time, DT-LoRA uses only $15.7$~GB of steady-state memory, compared with $15.5$~GB for LoRA and DoRA, $15.6$~GB for AdaLoRA, and $32.1$~GB for LoRA-Pro. Thus, DT-LoRA delivers the best accuracy with only a $0.2$~GB ($1.3\%$) increase over LoRA, while avoiding the substantially larger memory footprint of LoRA-Pro.

\begin{figure}[H]
	\captionsetup{font=small,skip=2pt}
	\centering
	\includegraphics[width=0.85\linewidth]{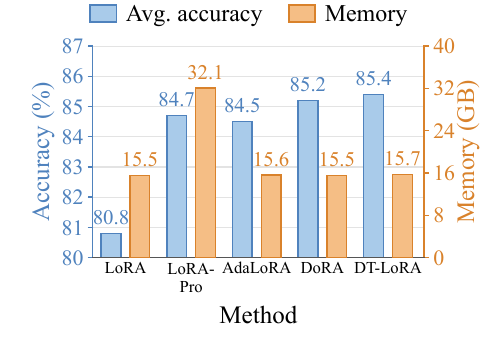}
	\caption{Average accuracy (\%) and training memory (GB) on LLaMA3-8B for commonsense reasoning.}
	\label{fig:memory-accuracy}
\end{figure}

\subsection{Sensitivity to the Refresh Interval}

We further analyze the effect of the refresh interval $\tau$ on gradient
coverage. As shown in Fig.~\ref{fig:refresh_interval}, a smaller refresh
interval generally leads to higher coverage. More frequent updates allow the
low-rank subspace to better track changes in the gradient directions during
training. The differences between refresh intervals emerge early and persist
throughout training. In particular, longer intervals lead to lower coverage
levels, suggesting that infrequent refreshes limit the subspace's ability to
capture evolving gradient directions. Although $\tau=1$ achieves the highest
coverage, refreshing at every step introduces additional computational cost.
In comparison, $\tau=5$ maintains relatively high coverage with fewer updates.
We therefore use $\tau=5$ as the default setting to balance gradient coverage
and training efficiency.

\begin{figure}[H]
	\captionsetup{font=small,skip=2pt}
	\centering
	\includegraphics[width=0.85\linewidth]{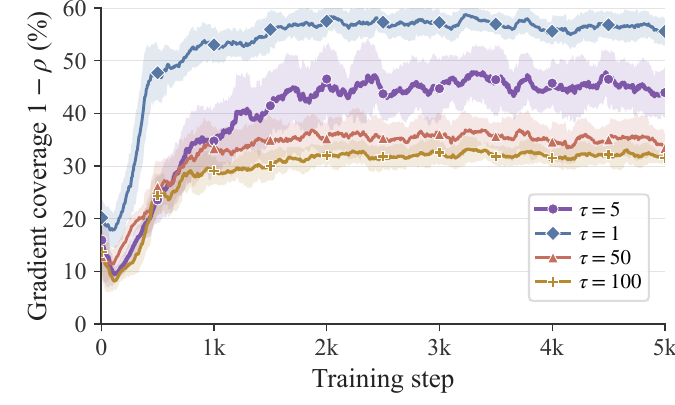}
	\caption{Effect of the refresh interval $\tau$ on gradient coverage for OVDG-SS at a fixed total rank of $r=16$.}
	\label{fig:refresh_interval}
\end{figure}

\subsection{Weight Trajectory Analysis}
\label{app:weight-trajectory}

We analyze the training dynamics of the two branches via weight trajectories. We adopt the relative displacement
$\bar{R}(H) = \left\langle \frac{\|W(t)-W(t-H)\|_F}{\|W(t-H)\|_F} \right\rangle$
to quantify how much each branch deviates from its state $H$ steps earlier, where $\langle\cdot\rangle$ denotes time averaging. As shown in \cref{fig:relative-displacement}, to reach the same normalized displacement level (0.15), the Stable branch needs $1.98\times$ more steps than the Dynamic branch, and its displacement remains lower across all windows. This shows the Dynamic branch explores the parameter space faster. The Dynamic branch’s fast displacement comes from periodic refresh, which reorients update directions. Such distinct update rates validate our dual-branch design: the Stable branch accumulates long-term update directions, while the Dynamic branch quickly captures transient residual structure.

\begin{figure}[H]
	\centering
	\includegraphics[width=0.9\linewidth]{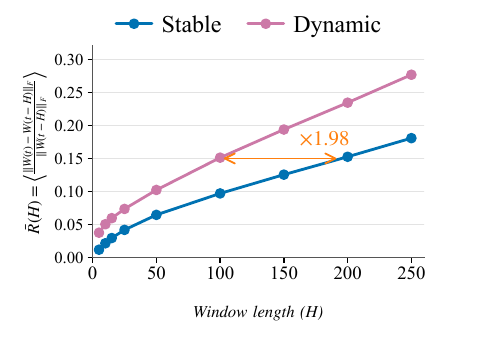}
	\caption{Relative displacement of Stable and Dynamic branches during training.}
	\label{fig:relative-displacement}
\end{figure}

\subsection{Low-Rank Approximation Accuracy}
\label{app:rsvd-accuracy}

We analyze the effects of oversampling $p$ and power iterations $m$ on the approximation accuracy of Matrix-Free Randomized SVD, with $q=4$. We evaluate residual operators from four early refresh steps in a single MS COCO run using 10 paired random seeds. We measure the relative Frobenius error between the approximate descent direction and its exact counterpart $-[G_{\mathrm{res}}]_q$. Lower error indicates higher approximation accuracy. We summarize the errors by taking the median over seeds for each operator, followed by the median across operators. As shown in \cref{fig:rsvd-accuracy}, with $m=2$, increasing $p$ from 0 to 2 and 4 reduces this error from $4.91\%$ to $0.59\%$ and $0.09\%$, respectively. Similarly, with $p=4$, increasing $m$ from 0 to 1 and 2 reduces the error from $13.89\%$ to $0.80\%$ and $0.09\%$, respectively. Further increases in either parameter yield smaller absolute reductions in error. These results show that a small oversampling width and few power iterations suffice to closely approximate the exact descent direction $-[G_{\mathrm{res}}]_q$.

\begin{figure}[H]
	\captionsetup{font=small,skip=2pt}
	\centering
	\begin{subfigure}[t]{0.49\linewidth}
		\centering
		\includegraphics[width=\linewidth]{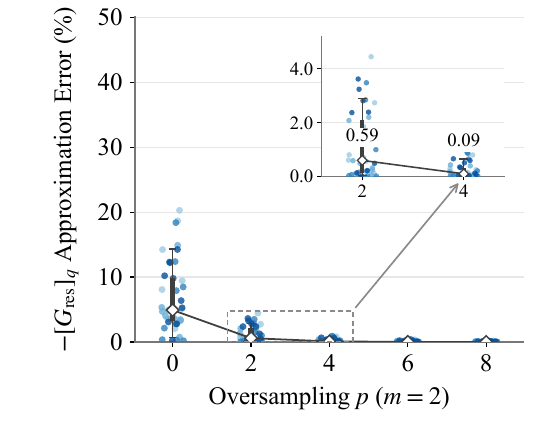}
		\caption{Effect of Oversampling}
		\label{fig:rsvd-p-sensitivity}
	\end{subfigure}\hfill
	\begin{subfigure}[t]{0.49\linewidth}
		\centering
		\includegraphics[width=\linewidth]{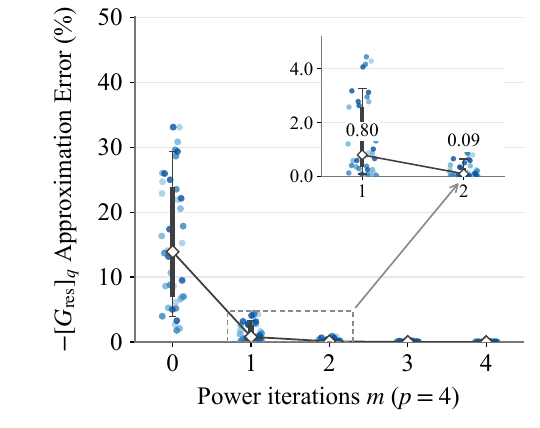}
		\caption{Effect of Power Iterations}
		\label{fig:rsvd-m-sensitivity}
	\end{subfigure}
\caption{Low-rank approximation accuracy: effects of (a) oversampling $p$ with $m=2$ and (b) power iterations $m$ with $p=4$. Dots show per-operator median errors over 10 paired seeds. Lines, thick bars, and whiskers indicate the median, interquartile range, and 10th--90th percentiles across operators, respectively.}
	\label{fig:rsvd-accuracy}
\end{figure}

\subsection{Full Task Results}

 \Cref{tab:cs7-full,tab:gta7-full} report the complete OVDG-SS results with EVA02 ViT-B/16. \Cref{tab:commonsense-full} reports the complete commonsense reasoning results on LLaMA2-7B.

\section{Dataset Details}
\label{app:datasets}

\begin{table*}[t]
	\centering
	\caption{Statistics of the VTAB-1K datasets.}
	\label{tab:vtab-dataset-details}
	\scriptsize
	\setlength{\tabcolsep}{1.8pt}
	\renewcommand{\arraystretch}{1.05}
	\resizebox{\textwidth}{!}{%
		\begin{tabular}{l | c c c c c c c | c c c c | c c c c c c c c}
			\hline
			& \multicolumn{7}{c|}{Natural} & \multicolumn{4}{c|}{Specialized} & \multicolumn{8}{c}{Structured} \\
			\cline{2-20}
			Dataset
			& \rotatebox{90}{CIFAR-100} & \rotatebox{90}{Caltech101} & \rotatebox{90}{DTD} & \rotatebox{90}{Flower102}
			& \rotatebox{90}{Pets} & \rotatebox{90}{SVHN} & \rotatebox{90}{Sun397}
			& \rotatebox{90}{Camelyon} & \rotatebox{90}{EuroSAT} & \rotatebox{90}{Resisc45} & \rotatebox{90}{Retinopathy}
			& \rotatebox{90}{Clevr-Count} & \rotatebox{90}{Clevr-Dist} & \rotatebox{90}{DMLab} & \rotatebox{90}{KITTI-Dist}
			& \rotatebox{90}{dSpr-Loc} & \rotatebox{90}{dSpr-Ori} & \rotatebox{90}{sNORB-Azim} & \rotatebox{90}{sNORB-Ele} \\
			\hline
			\# Classes & 100 & 102 & 47 & 102 & 37 & 10 & 397 & 2 & 10 & 45 & 5 & 8 & 6 & 6 & 4 & 16 & 16 & 18 & 9 \\
			Test & 10000 & 6084 & 1880 & 6149 & 3669 & 26032 & 21750 & 32768 & 5400 & 6300 & 42670 & 15000 & 15000 & 22735 & 711 & 73728 & 73728 & 12150 & 12150 \\
			\hline
			Train / Val. & \multicolumn{19}{c}{800 / 200 (1,000 for final retraining)} \\

			\hline
		\end{tabular}%
	}
	
\end{table*}

\subsection{VTAB-1K}
\label{app:vtab-details}

VTAB-1K consists of 19 classification tasks grouped into Natural, Specialized, and Structured categories. We fine-tune ImageNet-21K-pre-trained ViT-B/16 on each task. The number of classes and the official test-set sizes are reported in~\cref{tab:vtab-dataset-details}. Following the standard VTAB-1K protocol, we use 800 examples for training and 200 for validation during hyperparameter selection, then retrain on all 1,000 examples and report top-1 accuracy on the official test set.

\subsection{MS COCO}
\label{app:coco-details}
We follow the MS COCO protocol used by LieRA~\cite{si2025liera}, training on train2017 and evaluating on val2017. The split statistics are summarized in~\cref{tab:coco-dataset-details}. We fine-tune the ConvNeXt-V2-B backbone within Mask R-CNN. For detection and instance segmentation, we report the standard box- and mask-based metrics, including mAP, AP$_{50}$, AP$_{75}$, and AP$_{S}$, AP$_{M}$, and AP$_{L}$. The Avg. values in the main text are the means of the six reported metrics for each task.

\begin{table}[H]
	\centering
	\caption{Statistics of the MS COCO splits.}
	\label{tab:coco-dataset-details}
	\footnotesize
	\setlength{\tabcolsep}{6pt}
	\renewcommand{\arraystretch}{1.05}
	\begin{tabular}{l c c c}
		\hline
		Split & Images & Classes & Usage \\
		\hline
		train2017 & 118,287 & 80 & Training \\
		val2017 & 5,000 & 80 & Evaluation \\
		\hline
	\end{tabular}
\end{table}

\subsection{OVDG-SS}
\label{app:ovdg-dataset-details}

We follow the OVDG-SS benchmark protocol of S$^{2}$-Corr~\cite{zhao2026s2corr}. We fine-tune EVA02-CLIP ViT-B/16 and ViT-L/14 in CAT-Seg, and ViT-L/14 in PCA-Seg. \Cref{tab:ovdg-dataset-details} summarizes the source and target domains. CS-7 and GTA-7 define the Real-to-Real and Synthetic-to-Real training settings, respectively, and share seven source categories: road, sidewalk, building, vegetation, sky, person, and car. The target domains are grouped into Dv-19 and Dv-58. Dv-19 contains ACDC-19, BDD-19, and Mapi-19, together with CS-19 under the Synthetic-to-Real setting, whereas Dv-58 contains ACDC-41, BDD-41, Mapi-30, and RW-10. All target domains are evaluated using mean intersection over union (mIoU).

\begin{table}[H]
	\centering
	\caption{Source and target datasets for OVDG-SS.}
	\label{tab:ovdg-dataset-details}
	\footnotesize
	\setlength{\tabcolsep}{3pt}
	\renewcommand{\arraystretch}{1.1}
	\resizebox{\columnwidth}{!}{%
		\begin{tabular}{l l l r c c c}
			\hline
			Split & Dataset & Source & Images & Resolution & Classes & Group \\
			\hline
			Train & CS-7 & Cityscapes & 2,975 & $2048\times1024$ & 7 & -- \\
			Train & GTA-7 & GTA5 & 24,966 & $1914\times1052$ & 7 & -- \\
			\hline
			Test & CS-19 (GTA-7 only) & Cityscapes & 500 & $2048\times1024$ & 19 & Dv-19 \\
			Test & ACDC-19 & ACDC & 1,600 & $1920\times1080$ & 19 & Dv-19 \\
			Test & BDD-19 & BDD100K & 1,000 & $1280\times720$ & 19 & Dv-19 \\
			Test & Mapi-19 & Mapillary & 2,000 & 2K--4K & 19 & Dv-19 \\
			\hline
			Test & ACDC-41 & ACDC + SD 2.1 & 1,000 & $1920\times1080$ & 41 & Dv-58 \\
			Test & BDD-41 & BDD100K + SD 2.1 & 1,000 & $1280\times720$ & 41 & Dv-58 \\
			Test & Mapi-30 & Mapillary & 3,943 & 2K--4K & 30 & Dv-58 \\
			Test & RW-10 & ROADWork & 2,098 & $1280\times720$ & 10 & Dv-58 \\
			\hline
		\end{tabular}%
	}
\end{table}

\subsection{Commonsense Reasoning}
\label{app:commonsense-details}

Following LLM-Adapters~\cite{hu2023llm}, we fine-tune LLaMA2-7B and LLaMA3-8B separately on Commonsense170K, which combines the training sets of eight commonsense reasoning tasks. The task-level dataset statistics and answer formats are listed in \cref{tab:commonsense-dataset-details}. We evaluate each fine-tuned model on the test set of each task and report accuracy.

\begin{table}[H]
	\centering
	\caption{Statistics of the commonsense reasoning tasks.}
	\label{tab:commonsense-dataset-details}
	\footnotesize
	\setlength{\tabcolsep}{6pt}
	\renewcommand{\arraystretch}{1.05}
	\begin{tabular}{l c c c}
		\hline
		Task & Train & Test & Answer \\
		\hline
		BoolQ & 9.4K & 3,270 & Yes/No \\
		PIQA & 16.1K & 1,838 & Option \\
		SIQA (SocialIQA) & 33.4K & 1,954 & Option \\
		HellaS. (HellaSwag) & 39.9K & 10,042 & Option \\
		WinoG. (WinoGrande) & 63.2K & 1,267 & Option \\
		ARC-e (ARC-Easy) & 2.3K & 2,376 & Option \\
		ARC-c (ARC-Challenge) & 1.1K & 1,172 & Option \\
		OBQA (OpenBookQA) & 5.0K & 500 & Option \\
		\hline
	\end{tabular}
\end{table}

\subsection{GLUE}
\label{app:glue-details}

We follow the standard eight-task GLUE evaluation protocol used in prior parameter-efficient fine-tuning studies~\cite{hu2022lora,zhang2023adalora,liu2024dora}. \Cref{tab:glue-task-details} lists the task types, split sizes, and evaluation metrics. We fine-tune DeBERTaV3-base~\cite{he2023debertav3} independently on each task and report the mean development-set performance over three random seeds.
\begin{table}[H]
	\centering
	\caption{GLUE tasks and statistics.}
	\label{tab:glue-task-details}
	\footnotesize
	\setlength{\tabcolsep}{8pt}
	\renewcommand{\arraystretch}{1.1}
	\resizebox{\columnwidth}{!}{%
		\begin{tabular}{l l c c l}
			\hline
			Task & Type & Train & Dev & Metric \\
			\hline
			CoLA & Acceptability & 8.5k & 1.0k & Matthews Corr. \\
			MNLI & NLI & 393k & 20k & Accuracy \\
			MRPC & Paraphrase & 3.7k & 408 & Accuracy \\
			RTE & NLI & 2.5k & 277 & Accuracy \\
			QNLI & QA/NLI & 105k & 5.5k & Accuracy \\
			SST-2 & Sentiment & 67k & 872 & Accuracy \\
			STS-B & Similarity & 5.7k & 1.5k & Pearson Corr. \\
			QQP & Paraphrase & 364k & 40k & Accuracy \\
			\hline
		\end{tabular}%
	}
\end{table}
\FloatBarrier

\section{Hyperparameter Settings}
\label{app:hyperparameters}

The notation follows the main text throughout this section. Here, $r$ denotes the total rank, while $r_s$ and $r_d$ denote the ranks of the Stable and Dynamic branches. The parameter $\alpha$ is the LoRA scaling factor, and $\tau$ is the refresh interval measured in optimization steps. The ratios $\rho_{\mathrm{refresh}}$ and $\rho_{\mathrm{jump}}$ control the refresh magnitude and the maximum relative jump, respectively. When the jump constraint is violated, we reduce $\beta$ and recompress the update until the constraint is satisfied or the maximum of six backtracking steps is reached.

\subsection{Image Classification}
\label{app:vision-hyperparameters}

\Cref{tab:dtlora-vision-hyperparameters} summarizes the VTAB-1K configuration. We select the learning rate and weight decay from the listed candidate sets using the validation split, while keeping the remaining settings fixed across all 19 tasks.

\begin{table}[!htbp]
	\centering
	\caption{Hyperparameters of DT-LoRA on VTAB-1K.}
	\label{tab:dtlora-vision-hyperparameters}
	\footnotesize
	\setlength{\tabcolsep}{6pt}
	\renewcommand{\arraystretch}{1.0}
	\begin{tabular}{l c}
		\hline
		Hyperparameter & ViT-B/16 \\
		\hline
		Rank $r$ & 8 \\
		$\alpha$ & 8 \\
		Refresh Interval $\tau$ & 5 steps \\
		Oversampling $p$ & 4 \\
		Power Iterations $m$ & 2 \\
		$\rho_{\mathrm{refresh}}$ & 0.25 \\
		$\rho_{\mathrm{jump}}$ & 0.05 \\
		\hline
		Optimizer & AdamW \\
		Learning Rate & $\{10^{-3},5 \times 10^{-3}, 10^{-2}\}$ \\
		Weight Decay & $\{10^{-4}, 5 \times 10^{-4},10^{-3}\}$ \\
		Batch Size (train / eval) & 64 / 512 \\
		Warmup & 10 epochs, from $1\times10^{-6}$ \\
		\hline
	\end{tabular}
\end{table}

\subsection{Object Detection and Instance Segmentation}
\label{app:coco-hyperparameters}

\Cref{tab:dtlora-coco-hyperparameters} summarizes the configuration used for MS COCO object detection and instance segmentation. For all LoRA variants, adapters are inserted into every convolutional and linear layer while the original weights remain frozen. Convolutional weight tensors with more than two dimensions are reshaped into matrices for adaptation, and the resulting updates are reshaped back to the original tensor shape. The Dynamic branch is refreshed every $\tau$ steps starting from step 5, while the Stable branch remains unchanged during refreshes. The learning rate is linearly warmed up for the first 1,000 iterations and then decayed by a factor of 0.1 at epochs 8 and 11.

\begin{table}[!htbp]
	\centering
	\caption{Hyperparameters of DT-LoRA on MS COCO.}
	\label{tab:dtlora-coco-hyperparameters}
	\footnotesize
	\setlength{\tabcolsep}{6pt}
	\renewcommand{\arraystretch}{1.0}
	\begin{tabular}{l c}
		\hline
		Hyperparameter & ConvNeXt-V2-B \\
		\hline
		Adapted Modules & All Conv2d and Linear layers \\
		Rank $r$ ($r_s$ / $r_d$) & 16 (12 / 4) / 32 (24 / 8) \\
		$\alpha$ & 16 / 32 \\
		Refresh Interval $\tau$ & 5 steps \\
		Oversampling $p$ & 4 / 8 \\
		Power Iterations $m$ & 2 \\
		$\rho_{\mathrm{refresh}}$ & 0.25 \\
		$\rho_{\mathrm{jump}}$ & 0.05 \\
		\hline
		Optimizer & AdamW \\
		Learning Rate & $1\times10^{-4}$ \\
		Weight Decay & 0.05 \\
		Batch Size & 16 \\
		LR Schedule & Step ($\times0.1$ at epochs 8, 11) \\
		Warmup & Linear, 1,000 iters \\
		Epochs & 12 \\
		Mixed Precision & AMP \\
		\hline
	\end{tabular}
\end{table}

\subsection{OVDG-SS}
\label{app:ovdg-hyperparameters}

\Cref{tab:dtlora-ovdg-hyperparameters} summarizes the shared settings for OVDG-SS across both backbones, segmentation frameworks, and training settings. Only the number of adapted visual layers varies with encoder depth. DT-LoRA adapts the $Q$, $K$, $V$, and $O$ projections in all vision and text layers. The fused text-side QKV projection uses three separate adapters.

\begin{table}[H]
	\centering
	\caption{Hyperparameters of DT-LoRA on OVDG-SS.}
	\label{tab:dtlora-ovdg-hyperparameters}
	\footnotesize
	\setlength{\tabcolsep}{6pt}
	\renewcommand{\arraystretch}{1.0}
	\begin{tabular}{l c c}
		\hline
		Hyperparameter & ViT-B/16 & ViT-L/14 \\
		\hline
		Adapted Modules & \multicolumn{2}{c}{$Q$, $K$, $V$, $O$} \\
		Adapted Layers (visual / text) & 12 / 12 & 24 / 12 \\
		Rank $r$ ($r_s$ / $r_d$) & \multicolumn{2}{c}{16 (12 / 4)} \\
		$\alpha$ & \multicolumn{2}{c}{16} \\
		\hline
		Refresh Interval $\tau$ & \multicolumn{2}{c}{5 steps} \\
		Oversampling $p$ & \multicolumn{2}{c}{4} \\
		Power Iterations $m$ & \multicolumn{2}{c}{2} \\
		$\rho_{\mathrm{refresh}}$ & \multicolumn{2}{c}{0.25} \\
		$\rho_{\mathrm{jump}}$ & \multicolumn{2}{c}{0.05} \\
		Max Backtracks & \multicolumn{2}{c}{6} \\
		\hline
		Learning Rate (visual / text) & \multicolumn{2}{c}{$3\times10^{-5}$ / $1\times10^{-5}$} \\
		LR Schedule & \multicolumn{2}{c}{Cosine} \\
		Warmup & \multicolumn{2}{c}{Linear, 2,000 iters} \\
		Iterations & \multicolumn{2}{c}{16,000} \\
		\hline
	\end{tabular}
\end{table}
\subsection{Commonsense Reasoning}
\label{app:commonsense-hyperparameters}

\Cref{tab:dtlora-commonsense-hyperparameters} lists the configuration used for Commonsense170K. The same settings are used for both LLaMA2-7B and LLaMA3-8B. We adapt the attention and feed-forward projections with a fixed Stable/Dynamic rank split and use AdamW with a linear learning-rate schedule. The remaining optimization settings, including dropout, warmup, and the number of epochs, are given in the table.

\subsection{GLUE}
\label{app:glue-hyperparameters}

\Cref{tab:dtlora-glue-hyperparameters} summarizes the shared and task-specific settings for GLUE. We apply DT-LoRA to all attention and feed-forward projections in DeBERTaV3-base. The adapter configuration and optimizer settings are shared across tasks, whereas the learning rate and number of epochs are selected independently for each task.

\begin{table}[H]
	\centering
	\caption{Hyperparameters of DT-LoRA on commonsense reasoning.}
	\label{tab:dtlora-commonsense-hyperparameters}
	\footnotesize
	\setlength{\tabcolsep}{6pt}
	\renewcommand{\arraystretch}{1.0}
	\begin{tabular}{l c}
		\hline
		Hyperparameter & LLaMA2-7B / LLaMA3-8B \\
		\hline
		Adapted Modules & $Q$, $K$, $V$, Up, Down \\
		Rank $r$ ($r_s$ / $r_d$) & 16 (12 / 4) \\
		$\alpha$ & 16 \\
		Refresh Interval $\tau$ & 5 steps \\
		Oversampling $p$ & 4 \\
		Power Iterations $m$ & 2 \\
		$\rho_{\mathrm{refresh}}$ & 0.25 \\
		$\rho_{\mathrm{jump}}$ & 0.05 \\
		\hline
		Optimizer & AdamW \\
		Learning Rate & $1\times10^{-4}$ \\
		Batch Size & 16 \\
		LR Schedule & Linear \\
		Dropout & 0.05 \\
		Warmup & 100 steps \\
		Epochs & 3 \\
		\hline
	\end{tabular}
\end{table}

\begin{table}[H]
	\centering
	\caption{Hyperparameters of DT-LoRA on GLUE.}
	\label{tab:dtlora-glue-hyperparameters}
	\footnotesize
	\setlength{\tabcolsep}{8pt}
	\renewcommand{\arraystretch}{1.00}
	\resizebox{\columnwidth}{!}{%
		\begin{tabular}{l c c c c}
			\hline
			\multicolumn{5}{c}{\textbf{Shared Settings}} \\
			\hline
			\multicolumn{2}{l}{Adapted Modules} & \multicolumn{3}{l}{$Q$, $K$, $V$, $O$, FFN-Int, FFN-Out} \\
			\multicolumn{2}{l}{Rank $r$ ($r_s$ / $r_d$)} & \multicolumn{3}{l}{8 (6 / 2)} \\
			\multicolumn{2}{l}{$\alpha$} & \multicolumn{3}{l}{16} \\
			\multicolumn{2}{l}{Refresh Interval $\tau$} & \multicolumn{3}{l}{5 steps} \\
			\multicolumn{2}{l}{Oversampling $p$} & \multicolumn{3}{l}{2} \\
			\multicolumn{2}{l}{Power Iterations $m$} & \multicolumn{3}{l}{2} \\
			\multicolumn{2}{l}{$\rho_{\mathrm{refresh}}$} & \multicolumn{3}{l}{0.25} \\
			\multicolumn{2}{l}{$\rho_{\mathrm{jump}}$} & \multicolumn{3}{l}{0.05} \\
			\multicolumn{2}{l}{Optimizer} & \multicolumn{3}{l}{AdamW} \\
			\multicolumn{2}{l}{Batch Size} & \multicolumn{3}{l}{32} \\
			\multicolumn{2}{l}{LR Schedule} & \multicolumn{3}{l}{Linear decay with 10\% warmup} \\
			\multicolumn{2}{l}{Gradient Clipping} & \multicolumn{3}{l}{max\_grad\_norm = 1.0} \\
			\hline
			\hline
			\multicolumn{5}{c}{\textbf{Task-Specific Settings}} \\
			\hline
			Task & Learning Rate & Epochs & Max Seq. Len. & \\
			\hline
			CoLA & $8\times10^{-4}$ & 12 & 64 & \\
			MNLI & $5\times10^{-4}$ & 10 & 256 & \\
			MRPC & $1\times10^{-3}$ & 28 & 256 & \\
			RTE & $1.2\times10^{-3}$ & 5 & 512 & \\
			QNLI & $5\times10^{-4}$ & 5 & 256 & \\
			SST-2 & $8\times10^{-4}$ & 1 & 256 & \\
			STS-B & $5\times10^{-4}$ & 5 & 256 & \\
			QQP & $5\times10^{-4}$ & 5 & 320 & \\
			\hline
		\end{tabular}%
	}
\end{table}

\FloatBarrier

\end{document}